\documentclass[sn-mathphys]{JORSC-jnl}% Math and Physical Sciences Reference Style
\jyear{2023}%

\theoremstyle{thmstyleone}%
\newtheorem{theorem}{Theorem}%  meant for continuous numbers
\newtheorem{remark}{Remark}%

\theoremstyle{thmstylethree}%
\newtheorem{definition}{Definition}%
\newtheorem{lemma}{Lemma}
\usepackage{indentfirst}
\usepackage{graphicx}%
\usepackage{multirow}%
\usepackage{amsmath,amssymb,amsfonts}%
\usepackage{amsthm}%
\usepackage{mathrsfs}%
\usepackage[title]{appendix}%
\usepackage{xcolor}%
\usepackage{textcomp}%
\usepackage{manyfoot}%
\usepackage{booktabs}%
\usepackage{algorithm}%
\usepackage{algorithmicx}%
\usepackage{algpseudocode}%
\usepackage{listings}%
\usepackage[utf8]{inputenc}  % 处理特殊字符
\usepackage[T1]{fontenc}     % 支持欧洲字符集
\usepackage{bbding}%错号
\usepackage{pifont}%错号 
\usepackage{txfonts}
\usepackage{paralist}
\usepackage[misc]{ifsym}
\usepackage{epsfig}
\usepackage{epstopdf} 
\usepackage{bm,bbm}    
\usepackage{stfloats}
\usepackage{url}
\usepackage{verbatim}
\usepackage{makecell} 
\usepackage{tabularx}
\usepackage{supertabular} 
\usepackage{subcaption}
\usepackage{array}%调整列宽，并单元格内文本居中
\usepackage[demo,export]{adjustbox}

\hypersetup{
    colorlinks=true,
    linkcolor=blue,
    citecolor=blue,
    urlcolor=blue
}
\theoremstyle{thmstyleone} % 
\newtheorem{assumption}{Assumption} 
\begin{document}
\renewcommand{\qedsymbol}{}
\title[Heaviside Low-Rank Support Matrix Machine]{Heaviside Low-Rank Support Matrix Machine}

\author[1]{\fnm{Xian-Chao} \sur{Xiu}}\email{xcxiu@shu.edu.cn}

\author[1]{\fnm{Sheng-Hao} \sur{Sun}}\email{m15834069451@163.com}

\author*[2]{\fnm{Xin-Rong} \sur{Li}}\email{lixinrong@mail.neu.edu.cn}

\author[3]{\fnm{Ji-Yuan} \sur{Tao}}\email{jtao@loyola.edu}

\affil[1]{\orgdiv{School of Mechatronic Engineering and Automation}, \orgname{Shanghai University}, \orgaddress{\city{Shanghai}, \postcode{200400}, \country{China}}}

\affil*[2]{\orgdiv{National Frontiers Science Center for Industrial Intelligence and Systems Optimization}, \orgname{Northeastern University}, \orgaddress{\city{Shenyang}, \postcode{110819}, \country{China}}}

\affil[3]{\orgdiv{Department of Mathematics and Statistics}, \orgname{Loyola University Maryland}, \orgaddress{\city{Baltimore}, \postcode{MD 21210}, \country{USA}}}

\equalcont{This paper was supported in part by the National Natural Science Foundation of China under Grant Nos. 12371306, 12501420.}


\abstract{Support matrix machine (SMM) is an emerging classification framework that directly handles matrix-structured observations, thereby avoiding the spatial correlations destroyed by vectorization.
However, most existing SMM variants rely on convex or nonconvex surrogate loss functions, which may lead to high sensitivity to noise.
To address this issue, we propose a novel Heaviside low-rank SMM model called HL-SMM, which leverages the Heaviside loss instead of the common hinge or ramp losses for robustness. 
Moreover, the low-rank constraint is adopted to accurately characterize the inherent global structure.
In theory, we analyze the Karush-Kuhn-Tucker (KKT) points and rigorously prove the sufficient and necessary conditions.
In algorithms, we develop an effective proximal alternating minimization (PAM) scheme, where all subproblems have closed-form solutions.
Extensive experiments on benchmark datasets validate that the proposed HL-SMM achieves superior classification accuracy and robustness compared to state-of-the-art methods.}



\maketitle

\section{Introduction}\label{sec1}
Classification is a fundamental task in machine learning and pattern recognition, with extensive applications in signal and image processing. For example, in the medical field, classification has been shown to remarkably enhance the accuracy and efficiency of diagnosing various diseases \cite{yan2023reconfigurable}. Traditional classification methods mainly include support vector machine (SVM) \cite{cortes1995support}, k-nearest neighbors (KNN) \cite{cover1967nearest}, and logistic regression (LR) \cite{lavalley2008logistic}. Among these, SVM stands out as a supervised learning method with rigorous mathematical foundations and statistical learning theory, see, e.g., \cite{cervantes2020comprehensive,tanveer2024comprehensive} for recent surveys.

Over the past few decades, SVM has evolved into a series of variants, which can be categorized into linear SVM \cite{chen2026rml}, kernel SVM \cite{zhang2025stochastic}, and deep SVM \cite{okwuashi2020deep}.
By maximizing the separation margin between different classes, SVM exhibits strong generalization ability and excellent performance in processing high-dimensional feature spaces. 
However, these SVM variants deal with input data represented in the vector form. 
In reality, most data in the world exists in the matrix form, such as medical images and facial images.
As described in \cite{kumari2025support}, converting matrices to vectors not only destroys the inherent spatial correlations but also significantly increases the computational complexity.

Recently, the support matrix machine (SMM) was proposed and studied \cite{luo2015support}. 
It directly processed matrix data with the help of the hinge loss and nuclear norm, ensuring the effectiveness of classification while making full use of the data structure. 
After that, numerous variants have been developed to enhance its efficiency and generalization.
For example, Liang et al. \cite{liang2022adaptive} constructed the least-squares SMM (LS-SMM) by replacing the hinge loss with the least-squares loss for matrix-structured EEG classification. 
Li et al. \cite{li2022highly} introduced non-parallel hyperplanes and used small infrared thermal images for fault diagnosis. 
Li et al. \cite{li2022auto} improved the generalization performance of SMM by extending the input matrix through an auto-correlation function transformation.
In fact, the hinge loss function used in SMM is a convex approximation of the Heaviside loss. While this simplifies computation, it makes the method more sensitive to noise. 
To address this limitation, Feng et al. \cite{feng2022support} developed a new method with the pinball loss (Pin-SMM) that better captured the underlying data distribution and significantly enhanced its resistance to noise near the decision boundary.
Other variants, such as incorporating the truncated hinge loss \cite{qian2019robust} and the truncated pinball loss \cite{li2024support}, are also introduced.
Fig. \ref{fig1} visualizes the coefficient matrices obtained by different loss functions. 
The more localized and cleaner patterns indicate that the loss can better suppress noise and outliers. 
It can be seen that the Heaviside loss performs better than others.

\begin{figure*}[t]
    \centering

    \begin{subfigure}[b]{0.45\textwidth} 
        \centering

        \includegraphics[width=0.7\linewidth]{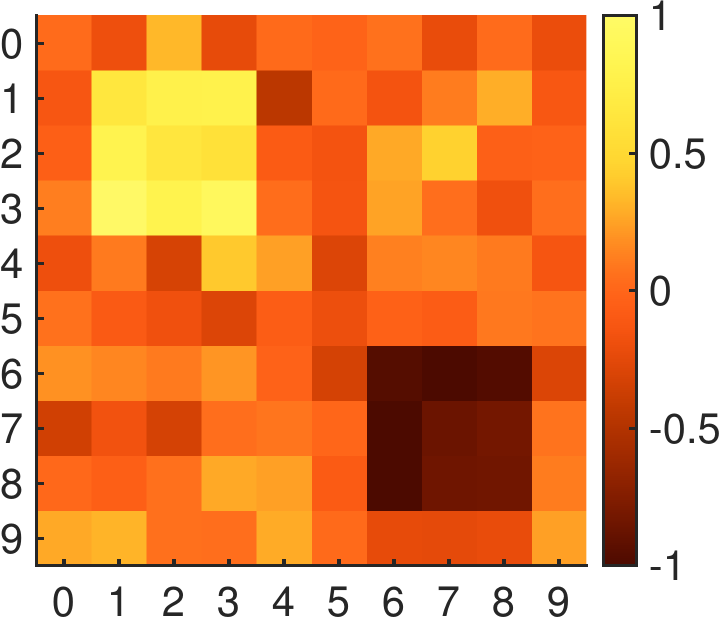} 
        \caption{Hinge loss}
    \end{subfigure}%
        \hspace{-0.5cm}
    \begin{subfigure}[b]{0.45\textwidth}
        \centering
        \includegraphics[width=0.7\linewidth]{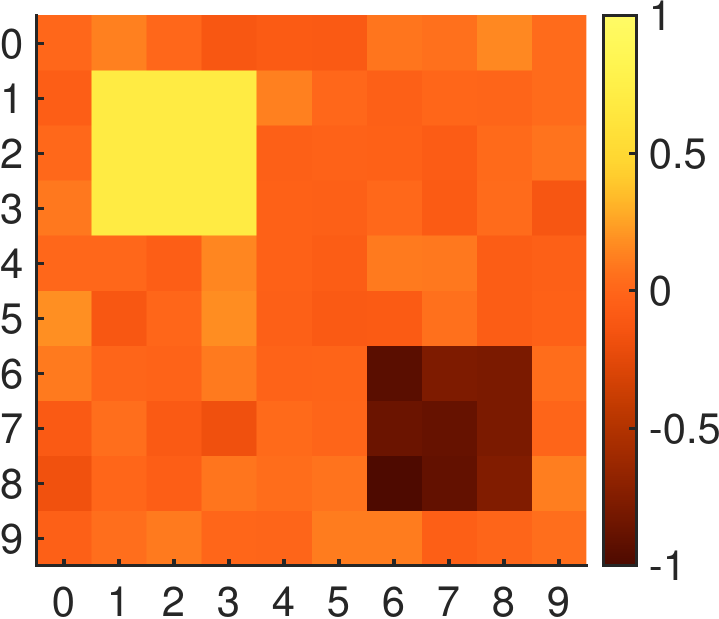}
        \caption{Pinball loss}
    \end{subfigure}%

%    \vspace{0.3cm} 

    \begin{subfigure}[b]{0.45\textwidth}
        \centering
        \includegraphics[width=0.7\linewidth]{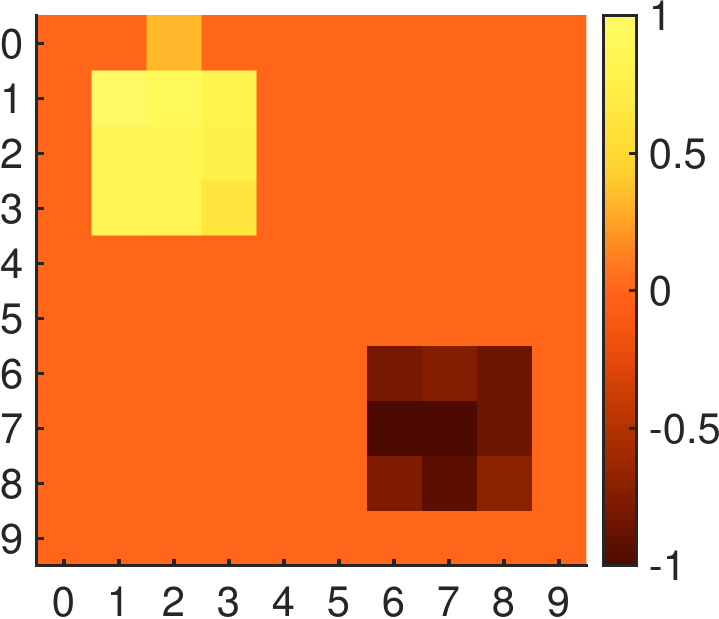}
        \caption{Truncated pinball loss}
    \end{subfigure}%
        \hspace{-0.5cm}
    \begin{subfigure}[b]{0.45\textwidth}
        \centering
        \includegraphics[width=0.7\linewidth]{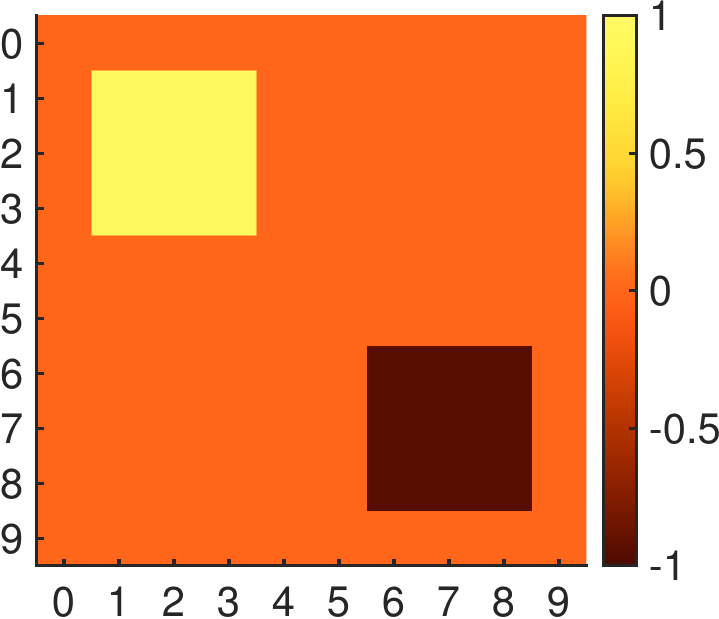}
        \caption{Heaviside loss}
    \end{subfigure}%

    \caption{Coefficient matrices obtained by different loss functions.}
    \label{fig1}
\end{figure*}

Although these aforementioned SMM methods have obtained excellent noise resistance, the inherent sparsity and low-rank characteristics of the input data were not fully considered, potentially leading to poor performance when dealing with complex intrinsic structures and data volumes \cite{xiu2022efficient,xiu2025bi}.
Zheng et al. \cite{zheng2018sparse} integrated the $\ell_1$-norm with SMM to filter out redundant features, thus achieving more interpretable modeling.
Wang et al. \cite{wang2022sparse} employed a twin non-parallel hyperplane structure with the $\ell_1$-norm regularization to ensure sparsity and discriminativeness.
Unlike the previous two methods, Gu et al. \cite{gu2021ramp} introduced the nonconvex ramp loss function to SMM called Ramp-SMM.
Note that the rank in the objective function is often replaced with the nuclear norm for overcoming computational difficulties \cite{li2023normal}.
However, this often tends to over-shrink singular values, and thus distort the underlying low-rank structure, especially when the true intrinsic dimension is small \cite{huang2018rank}.
More importantly, as stated in \cite{zhang2013counterexample}, the solution to the nuclear norm minimization problem is not necessarily the solution to the corresponding rank minimization problem.
\textit{A natural question arises: can we construct a new SMM variant that combines the Heaviside loss for robustness and the rank constraint for structure preservation?}

This article provides a definitive answer and proposes a rank-constrained SMM with the Heaviside loss, termed HL-SMM. 
Specifically, the Heaviside loss improves robustness to noise and outliers, while the rank constraint explicitly enforces the low-dimensional representation. 
In fact, to our best knowledge, this is the first SMM variant that introducing the Heaviside loss. 
However, the resulting problem is nonconvex and nonsmooth. 
To this end, we establish the necessary and sufficient conditions for KKT points and develop an effective PAM algorithm with closed-form updates for each subproblem. 
Furthermore, we conduct extensive experiments on six benchmark datasets, including robustness evaluation and parameter sensitivity analysis, to evaluate the performance of our proposed method.

The remainder of this paper is organized as follows.
Section \ref{sec2} introduces the preliminaries.
Section \ref{sec3} presents the proposed model and optimization algorithm.
Section \ref{sec4} reports the experimental results and discussions.
Finally, Section \ref{sec5} concludes this paper.

\section{Preliminaries}\label{sec2}

In this paper, matrices are represented by bold uppercase letters, vectors by bold lowercase letters, and scalars by lowercase letters. Let $\mathbb{R}^{p \times q}$ and $\mathbb{R}^p$ be the sets of $p \times q$-dimensional matrices and $p$-dimensional vectors, respectively. 
For a matrix $\mathbf{W}\in\mathbb{R}^{p\times q}$, denote its transpose as $\mathbf{W}^\top$.
The inner product of two matrices $\mathbf{W},\mathbf{Y}\in\mathbb{R}^{p\times q}$ is
$\langle\mathbf{W},\mathbf{Y}\rangle = \operatorname{tr}(\mathbf{W}^\top\mathbf{Y}),$
where $\operatorname{tr}(\cdot)$ denotes the trace.
The Frobenius norm of $\mathbf{W}$ is $\|\mathbf{W}\|_{\mathrm{F}} = \sqrt{\langle\mathbf{W},\mathbf{W}\rangle}$, and the nuclear norm is $\|\mathbf{W}\|_*$. 
In addition, the singular value decomposition (SVD) is denoted as $
\mathbf{W}=\mathbf{U}\,\Sigma(\mathbf{W})\,\mathbf{V}^\top$,
where $\mathbf{U}\in\mathcal{O}^p$ and $\mathbf{V}\in\mathcal{O}^q$ are the orthogonal matrices, $\Sigma(\mathbf{W})$ is the diagonal matrix, whose diagonal elements are
called the singular values and reoresented by $\sigma_i(\mathbf{W})$.
Note that we let $\mathbf{W}_i$ be the $i$‑th matrix sample within a given set of matrices.

For a vector $\mathbf{z}\in\mathbb{R}^p$, the $i$‑th element is denoted as $(\mathbf{z})_i$.
Define the Euclidean norm as $\|\mathbf{z}\|$, the $\ell_0$-norm as $\|\mathbf{z}\|_0$, which counts the number of
nonzero elements in the vector $\mathbf{z}$. 
For a scalar $t \in \mathbb{R}$, define $[t]_+ = \max\{t, 0\}$, thus $\mathbf{z}_+ = ([(\mathbf{z})_1]_+, \cdots, [(\mathbf{z})_p]_+)^\top$. The Heaviside loss is
\begin{equation*}
\|\mathbf{z}_+\|_0=\sum_{i=1}^m \ell_{0/1}((\mathbf{z})_i), \quad
\ell_{0/1}((\mathbf{z})_i)=
\begin{cases}
1,& (\mathbf{z})_i>0,\\
0,& (\mathbf{z})_i\le 0.
\end{cases}
\end{equation*}

Let $\mathscr{A}:\mathbb{R}^{p\times q}\rightarrow\mathbb{R}^m$ be the linear operator defined by
\begin{equation*}
\mathscr{A}(\mathbf{W})=
(\langle \mathbf{A}_1,\mathbf{W}\rangle,\cdots,\langle \mathbf{A}_m,\mathbf{W}\rangle)^\top,
\end{equation*}
where each $\mathbf{A}_i\in\mathbb{R}^{p\times q}$ is the given matrix. The adjoint operator $\mathscr{A}^*:\mathbb{R}^m\rightarrow\mathbb{R}^{p\times q}$ is defined by 
\begin{equation*}
\langle \mathscr{A}(\mathbf{W}),\bm{\lambda}\rangle
=
\langle \mathbf{W},\mathscr{A}^*(\bm{\lambda})\rangle,\quad
\forall\,\mathbf{W}\in\mathbb{R}^{p\times q}, \bm{\lambda}\in\mathbb{R}^m.
\end{equation*}
It follows directly that
\begin{equation*}
\mathscr{A}^*(\bm{\lambda})=\sum_{i=1}^m (\bm\lambda)_i \mathbf{A}_i.
\end{equation*}
We will frequently use the affine constraint
\begin{equation*}
\mathbf{z}=\mathbf{1}+\mathscr{A}(\mathbf{W}),
\end{equation*}
and  the affine set
\begin{equation*}
\mathcal{L}=\{(\mathbf{W},\mathbf{z})\in\mathbb{R}^{p\times q}\times\mathbb{R}^m \mid \mathbf{1}+\mathscr{A}(\mathbf{W})=\mathbf{z}\},
\end{equation*}
where  $\mathbf{1}$ is the vector of all‑ones in $\mathbb{R}^m$.  Furthermore, define the rank-constrained set
\begin{equation*}
\mathcal{R}=\{\mathbf{W}\in\mathbb{R}^{p\times q} \mid \mathrm{rank}(\mathbf{W})\le r\},
\end{equation*}
where $r<\min\{p,q\}$ is the prescribed rank parameter. Then, the feasible region is $\mathcal{F}=\mathcal{L}\cap\mathcal{R}$.

In order to characterize the rank-constrained optimization, some explicit expressions related to the projection and the normal cone of $\mathcal{R}$ are given as follows, see \cite{li2023normal} for details.

\begin{lemma}[Projection of $\mathcal{R}$]
Given $\mathbf{W}\in\mathbb{R}^{p\times q}$ with singular values $\sigma_1(\mathbf{W})\ge\cdots\ge\sigma_s(\mathbf{W})>0$,
the (possibly set-valued) projection onto the rank-constrained set is
\begin{equation*}
\Pi_{\mathcal{R}}(\mathbf{W})
=
\{
\mathbf{U}\,\mathrm{Diag}(\sigma_1(\mathbf{W}),\cdots,\sigma_r(\mathbf{W}),0,\cdots,0)\,\mathbf{V}^\top \mid  (\mathbf{U},\mathbf{V})\in\mathcal{O}^{p,q}(\mathbf{W})
\},
\end{equation*}
when $\sigma_r(\mathbf{W})>\sigma_{r+1}(\mathbf{W})$ (or $\sigma_r(\mathbf{W})=0$), the projection is single-valued. 
\end{lemma}

From the above lemma,  it can be seen that projecting a matrix onto the set $\mathcal{R}$ is equivalent to retaining its first $r$ largest singular values and discarding the rest, which is essentially a hard thresholding operation on the singular values.
In the following lemma, we will give the regular normal cone to $\mathcal{R}$.

\begin{lemma}[Normal cone of $\mathcal{R}$] \label{def3}
Let $\mathbf{W}\in\mathcal{R}$ with $\mathrm{rank}(\mathbf{W})=s\le r$ and SVD as above. The regular normal cone to $\mathcal{R}$
at $\mathbf{W}$ admits the representation
\begin{equation*}
\mathrm{N}_{\mathcal{R}}(\mathbf{W})
=
\begin{cases}
\displaystyle
\bigl\{
\mathbf{U}_{\Gamma_p^\perp} \mathbf{D} \mathbf{V}_{\Gamma_q^\perp}^\top \mid
\mathbf{D}\in\mathbb{R}^{(p-r)\times(q-r)}
\bigr\},
& s = r,
\\[4pt]
\{ \mathbf{0} \},
& s < r,
\end{cases}
\end{equation*}
where $\Gamma_p$ is the index set of nonzero singular values, $\Gamma_p^\perp=\{1,\ldots,p\}\setminus\Gamma_p$, $\mathbf{U}_{\Gamma_p^\perp}$  is the submatrix of $\mathbf{U}$ indexed by $\Gamma_p^\perp$, and $\mathbf{V}_{\Gamma_q^\perp}$  is the submatrix of $\mathbf{V}$ indexed by $\Gamma_q^\perp$.
\end{lemma}

At the end of this section, the proximal operator of $\|(\cdot)_+\|_0$ is described, see \cite{wang2021support2} for details.

\begin{lemma}[Proximal operator of $\|(\cdot)_+\|_0$]
For $\gamma>0$, define the proximal mapping of the function $\gamma\|\mathbf{z}_+\|_0$ by
\begin{equation*}
\mathrm{Prox}_{\gamma \|(\cdot)_+\|_0}(\mathbf{x})
=
\arg\min_{\mathbf{z}\in\mathbb{R}^p}\quad
\{
\gamma\|\mathbf{z}_+\|_0+\frac{1}{2}\|\mathbf{z}-\mathbf{x}\|^2
\}.
\end{equation*}
The solution is
\begin{equation*}
[\mathrm{Prox}_{\gamma\|(\cdot)_+\|_0}(\mathbf{x})]_i
=
\begin{cases}
0,& 0<(\mathbf{x})_i\le \sqrt{2\gamma},\\
(\mathbf{x})_i,& (\mathbf{x})_i\le 0 ~\text{or}~(\mathbf{x})_i>\sqrt{2\gamma}.
\end{cases}
\end{equation*}
\end{lemma}
Here, $\mathrm{Prox}_{\gamma\|(\cdot)_+\|_0}$ performs the hard thresholding on the positive entries of $\mathbf{x}$, and this operator is separable across coordinates. 

%-------------------------------------------------
\section{The Proposed Method}\label{sec3}
%-------------------------------------------------

\subsection{New Formulation}
Given training data $(\mathbf{X}_i, y_i), i = 1,\cdots, m$, where $\mathbf{X}_i\mathbf\in\mathbb{R}^{p\times q}$ is the observation, $y_i\in \{-1 , 1\}$ is the label. The SMM classification is to find a hyperplane, i.e.,
\begin{equation*}
y_i = \langle \mathbf{W},\mathbf {X}_i\rangle+b,
\end{equation*}
such that the data with different labels can be separated by the hyperplane. The original SMM problem can be mathematically expressed as
\begin{equation}\label{RSMM}
\begin{aligned}
\min\limits_{\mathbf{W},b}\quad
 &\frac{1}{2}\langle\mathbf{W},\mathbf{W}\rangle \\
		\text{s.t.} \quad& y_i(\langle \mathbf{W},\mathbf {X}_i\rangle+b)\geq 1, i=1,\cdots, m,\\
		 & \text{rank}(\mathbf{W}) \leq r,
\end{aligned}
\end{equation}
where $r$ is a positive integer satisfying $r<\min\{p,q\}$. Problem \eqref{RSMM} is based on the assumption that the two types of data can be successfully separated by the hyperplane. However, in practice, this is usually not the case, and the rank constraint is a nonconvex problem, making the optimization generally challenging.  A more practical and popular method is the regularized penalty model, which incorporates the nuclear norm relaxation, defined as
\begin{equation}\label{SMM-relax}
\begin{aligned}
\min_{\mathbf{W}, b} \quad
& \frac{1}{2}\langle\mathbf{W},\mathbf{W}\rangle  + \alpha \|\mathbf{W}\|_* + \beta\sum\limits_{i=1}^m \phi\left[1-y_i(\langle \mathbf W,\mathbf X_i\rangle+b)\right],\\
\end{aligned}
\end{equation}
where $\alpha, \beta > 0$ are the penalty parameters and $\phi$ corresponds to the loss function. For example, the classic Hinge-SMM \cite{luo2015support} adopts $\phi$ with the hinge loss, and in this case, problem \eqref{SMM-relax}  becomes
\begin{equation*}
\begin{aligned}
\min_{\mathbf{W}, b} \quad
&\frac{1}{2}\langle\mathbf{W},\mathbf{W}\rangle  + \alpha \|\mathbf{W}\|_* + \beta\sum\limits_{i=1}^m \left[1-y_i(\langle \mathbf W,\mathbf X_i\rangle+b)\right]_+.\\
\end{aligned}
\end{equation*}
Naturally, other loss functions, such as the pinball loss and the ramp loss, can also be considered. While these loss functions have demonstrated promising performance in existing studies, they still deviate from the intrinsic nature of the problem to some extent. Motivated by these observations, this paper introduces the Heaviside loss function and the rank constraint to precisely characterize this problem as 
\begin{equation}\label{RSMM-1}
	\begin{aligned}
		\min\limits_{\mathbf W,b}\quad
		&\frac{1}{2}\langle\mathbf W,\mathbf W\rangle+\beta\sum\limits_{i=1}^m  \ell_{0/1}\left[1-y_i(\langle \mathbf W,\mathbf X_i\rangle+b)\right]\\
		\text{s.t.} \quad & \text{rank}(\mathbf{W}) \leq r.
	\end{aligned}
\end{equation}

Let $\mathbf{z}=\mathbf{1}-\mathscr{A}(\mathbf W)-b\mathbf{y}$. According to Section \ref{sec2}, problem \eqref{RSMM-1} has the following equivalent form
\begin{equation}\label{RSMM-2} 
\begin{aligned}
\min\limits_{\mathbf W,\mathbf{z},b}\quad&\frac{1}{2}\langle\mathbf W,\mathbf W\rangle+\beta\|{\mathbf{z}}_+\|_0\\
{\rm s.t.}\quad &\mathbf{z}=\mathbf{1}-\mathscr{A}(\mathbf W)-b\mathbf{y},\\
\quad & \text{rank}(\mathbf{W}) \leq r.
\end{aligned}
\end{equation}
Without loss of generality, the bias term $b$ could be inserted into $\mathbf{W}$. Denote $\hat{\mathbf{W}} = [ \mathbf{W}^\top ~ b \mathbf{e}_j]^\top$,
where $\mathbf{e}_j \in \mathbb{R}^q$ is the $j$-th standard basis vector (a column vector with 1 at position $j$ and 0 elsewhere). 
Furthermore, define the linear operator ${\mathscr{A}}: \mathbb{R}^{(p+1) \times q} \to \mathbb{R}^m$ of $\hat{\mathbf{W}}$,  with a matrix $\mathbf{A}_i \in \mathbb{R}^{q \times (p+1)}$. For each training sample, it has $\mathbf{A}_i = -y_i [
\mathbf{X}_i^\top~\mathbf{e}_j]$. This construction ensures that
\begin{equation*}
\langle \hat{\mathbf{W}}, \mathbf{A}_i \rangle = -y_i \left( \langle \mathbf{W}, \mathbf{X}_i \rangle + b \right).
\end{equation*}
Therefore, the original affine constraint $\mathbf{z} = \mathbf{1} - \mathscr{A}(\mathbf{W}) - b\mathbf{y}$ can be rewritten as $  \mathbf{z} = \mathbf{1} + {\mathscr{A}}(\hat{\mathbf{W}})$, and thus problem \eqref{RSMM-2} can be further presented by
\begin{equation}\label{RSMM3} 
	\begin{aligned}
        \min\limits_{\hat{\mathbf{W}},\mathbf{z}}\quad
		&\frac{1}{2}\langle\hat{\mathbf{W}},\hat{\mathbf{W}}\rangle+\beta\|{\mathbf{z}}_+\|_0\\
		{\rm s.t.}\quad &\mathbf{1}+\mathscr{A}(\hat{\mathbf{W}})=\mathbf{z},\\
		\quad & \text{rank}(\hat{\mathbf{W}}) \leq r.\\
	\end{aligned}
\end{equation}

\subsection{Optimality Conditions}

The Lagrangian function of problem \eqref{RSMM3} is 
\begin{equation*}
\mathscr{L}(\hat{\mathbf{W}},\mathbf{z},\bm{\lambda})=\frac{1}{2}\langle\hat{\mathbf{W}},\hat{\mathbf{W}}\rangle+\beta\|{\mathbf{z}}_+\|_0+\langle \bm{\lambda}, {\mathbf{1}}+\mathscr{A}(\hat{\mathbf{W}})-\mathbf{z}\rangle,
\end{equation*}
where $\bm{\lambda}\in\mathbb{R}^m$ is the Lagrangian multiplier. Let $\nabla_{\hat{\mathbf{W}}}\mathscr{L}(\hat{\mathbf{W}},\mathbf{z},\bm{\lambda})$ and $\nabla_{\mathbf{z}}\mathscr{L}(\hat{\mathbf{W}},\mathbf{z},\bm{\lambda})$ stand for the gradients of $\mathscr{L}$ with respect to $\hat{\mathbf{W}}$ and $\mathbf{z}$, respectively,
where
\begin{equation*}
\begin{aligned}
\nabla_{\hat{\mathbf{W}}}\mathscr{L}(\hat{\mathbf{W}},\mathbf{z},\bm{\lambda})
&= \hat{\mathbf{W}}+\mathscr{A}^*(\bm{\lambda}),\\
\partial_{\mathbf{z}}\mathscr{L}(\hat{\mathbf{W}},\mathbf{z},\bm{\lambda})
&= \beta\partial\|\mathbf{z}_+\|_0 - \bm{\lambda}.
\end{aligned}
\end{equation*}

Recall that the function $\|(\cdot)_+\|_{0} : \mathbb{R}^{m} \rightarrow \mathbb{R}$ is lower semicontinuous on $\mathbb{R}^{m}$ \cite{pan2017optimality}. The regular subdifferentials of $\|\mathbf{z}_{+}\|_{0}$ at $\mathbf{z}$ enjoy the following property
$$ \partial\|\mathbf{z}_{+}\|_{0} = \left\{ \mathbf{d} \in \mathbb{R}^{m} \mid \begin{cases} 0, & (\mathbf{z})_{i} \neq 0\\ \geq 0, & (\mathbf{z})_{i} = 0 \end{cases}  \right\}.$$

Before discussing the optimality conditions,  denote 
\begin{equation*}
		\begin{aligned}
		\mathcal{L}& =\{(\hat{\mathbf{W}},\mathbf{z})\in\mathbb{R}^{(p+1)\times q}\times\mathbb{R}^m  \mid \mathbf{1}+\mathscr{A}(\hat{\mathbf{W}})=\mathbf{z} \},\\
		\mathcal{R}&=\{\hat{\mathbf{W}}\in\mathbb{R}^{(p+1)\times q}\mid \text{rank}(\hat{\mathbf{W}}) \leq r \}.
	\end{aligned}
	\end{equation*}
     Then, the feasible region of problem \eqref{RSMM3} is denoted by $\mathcal{F}=\mathcal{L}\cap \mathcal{R}$.

\begin{definition}\label{stat}
	A point $(\hat{\mathbf{W}},\mathbf{z})$ is called a Karush–Kuhn–Tucker (KKT) point of problem \eqref{RSMM3} if there exists a Lagrangian multiplier $\bm{\lambda}$ such that
	\begin{equation*}
		\begin{cases}
			-\partial_{(\hat{\mathbf{W}},\mathbf{z})}\mathscr{L}(\hat{\mathbf{W}},\mathbf{z},\bm{\lambda})
			\in \mathrm{N}_{\mathcal{F}}(\hat{\mathbf{W}},\mathbf{z}),\\[4pt]
			\mathrm{rank}(\hat{\mathbf{W}}) \leq r,\\[4pt]
			\mathbf{1}+\mathscr{A}(\hat{\mathbf{W}})-\mathbf{z}=0.
		\end{cases}
	\end{equation*}
\end{definition}

Given the SVD  $
\hat{\mathbf{W}}=\mathbf{U}\,\Sigma(\hat{\mathbf{W}})\,\mathbf{V}^\top$ and
matrices $\mathbf{A}_1,\cdots,\mathbf{A}_{m}\in\mathbb{R}^{(p+1)\times q}$, denote
\begin{equation*}\label{Ti}
	\mathbf{T}^i_{\hat{\mathbf{W}}} =
	\left[
	\begin{array}{cc}
		\mathbf{U}_{\Gamma}^\top \mathbf{A}_i \mathbf{V}_{\Gamma}
		& \mathbf{U}_{\Gamma}^\top \mathbf{A}_i \mathbf{V}_{\Gamma_q^\perp} \\[4pt]
		\mathbf{U}_{\Gamma_p^\perp}^\top \mathbf{A}_i \mathbf{V}_{\Gamma}
		& \mathbf{O}
	\end{array}
	\right],
	\quad
	\mathbf{R}^i_{\hat{\mathbf{W}}} = \mathbf{U}^\top \mathbf{A}_i \mathbf{V}_{\Gamma},
\end{equation*}
for $i=1,\cdots,m$, where $\Gamma$ is the index set of nonzero singular values.
Following a similar line as \cite{li2023normal}, Assumption \ref{ass1} can be obtained.

\begin{assumption}\label{ass1}
When $s=r$, the matrices $\mathbf{T}_{\hat{\mathbf{W}}}^i$, $i=1,2,\cdots,m$, are linearly independent. When $s<r$, the matrices $\mathbf{R}_{\hat{\mathbf{W}}}^i$, $i=1,2,\cdots,m$, are linearly independent.
\end{assumption}

The following theorem illustrates that under a constraint qualification tailored to the rank constraint (Assumption \ref{ass1}), local minimizers satisfy a KKT point condition.

\begin{theorem}[Necessary optimality condition] Suppose that Assumption \ref{ass1} holds. Then a local minimizer of problem \eqref{RSMM3} is a KKT point.
\end{theorem}

\begin{proof}
	Let $\mathbf{T}=(\hat{\mathbf{W}}, \mathbf{z})$ be a local minimizer of  problem $\eqref{RSMM3}$ and denote $f(\mathbf{T})$ as the objective function. Invoking the generalized Fermat's theorem, it derives that  
	$-\partial_\mathbf{T} f(\mathbf{T})\in\mathrm{N}_{\mathcal{F}}(\mathbf{T})$. Then it has
	\begin{equation}\label{T1}
		\begin{cases}
			-\nabla_{\hat{\mathbf{W}}}f(\mathbf{T})\in\mathrm{N}_{\mathcal{L}\cap \mathcal{R}}(\hat{\mathbf{W}}),\\
			-\partial_\mathbf{z} f(\mathbf{T})\in\mathrm{N}_{\mathcal{L}}(\mathbf{z}).
		\end{cases} 
	\end{equation}
Next, consider the following two cases, which correspond to the two scenarios of Assumption \ref{ass1} respectively.
\begin{description}
    \item[(Case 1)] If $\text{rank}(\hat{\mathbf{W}})=r$, it derives that  $\mathrm{N}_{\mathcal{L}\cap \mathcal{R}}(\hat{\mathbf{W}})=\mathrm{N}_{\mathcal{L}}(\hat{\mathbf{W}})+\mathrm{N}_{\mathcal{R}}(\hat{\mathbf{W}})$.  Then,	
\begin{equation*}\label{W1}
		-\nabla_{\hat{\mathbf{W}}} f(\mathbf{T})\in\mathrm{N}_{\mathcal{L}}(\hat{\mathbf{W}})+\mathrm{N}_{\mathcal{R}}(\hat{\mathbf{W}}).
	\end{equation*}

    \item[(Case 2)]     If $\text{rank}(\hat{\mathbf{W}})<r$, it shows that $\mathrm{N}_{\mathcal{L}\cap \mathcal{R}}(\hat{\mathbf{W}})=\mathrm{N}_{\mathcal{L}}(\hat{\mathbf{W}})$. Note that
	\begin{equation*}
		\mathrm{N}_{\mathcal{L}}(\hat{\mathbf{W}})=\{\sum_{i=1}^m(\bm{\lambda})_i y_i\mathbf{X}_i \mid (\bm{\lambda})_i\in\mathbb{R}, i=1,\cdots,m \}.
	\end{equation*}
    
	Relation \eqref{T1} indicates that there  exist $\bm{\lambda}$ such that $-\nabla_{\hat{\mathbf{W}}} f(\mathbf{T})+\sum_{i=1}^m(\bm{\lambda})_i y_i\mathbf{X}_i\in\mathrm{N}_{\mathcal{R}}(\hat{\mathbf{W}})$. Then,
	\begin{equation}\label{W2}
    -\partial_{\hat{\mathbf{W}}} \mathscr{L}(\mathbf{T},\bm{\lambda})= -\nabla_{\hat{\mathbf{W}}}f(\mathbf{T})+\sum_{i=1}^m(\bm{\lambda})_i y_i\mathbf{X}_i\in\mathrm{N}_{\mathcal{R}}(\hat{\mathbf{W}}).
    \end{equation}
\end{description}
Similarly, since $\mathrm{N}_{\mathcal{L}}(\mathbf{z})=\mathbb{R}^m$, relation \eqref{T1} implies that 
    	\begin{equation}\label{N1}
        -\partial_\mathbf{z} \mathscr{L}(\mathbf{T},\bm{\lambda})\in\mathrm{N}_{\mathcal{L}}(\mathbf{z}).
            \end{equation}
Combing \eqref{W2} and \eqref{N1}, it holds that
	\begin{equation*}
		\begin{cases}
			-\partial_{\hat{\mathbf{W}}} \mathscr{L}(\mathbf{T},\bm{\lambda})\in\mathrm{N}_{\mathcal{R}}(\hat{\mathbf{W}}),\\
			-\partial_\mathbf{z} \mathscr{L}(\mathbf{T},\bm{\lambda})\in\mathrm{N}_{\mathcal{L}}(\mathbf{z}).
		\end{cases} 
	\end{equation*} 	
Therefore, it has 
	\begin{equation*}
-\partial_{({\hat{\mathbf{W}},\mathbf{z}})}\mathscr{L}(\hat{\mathbf{W}},\mathbf{z},\bm{\lambda})
	\in\text{N}_{\mathcal{F}}(\hat{\mathbf{W}},\mathbf{z}).
	\end{equation*} 
From Definition \ref{stat},  it is concluded that $\mathbf{T}$ is a KKT point. The proof is completed.
\end{proof}

\begin{theorem}[Suffcient optimality condition] A KKT point $(\hat{\mathbf{W}}^*,\mathbf{z}^*)$ of problem \eqref{RSMM3} is also a local minimizer.
\end{theorem}

\begin{proof}
	Let $(\hat{\mathbf{W}}^*,\mathbf{z}^*)$ be a KKT point of problem \eqref{RSMM3} and $\bm{\lambda}^*\in\mathbb{R}^m$ be the corresponding Lagrangian multiplier, it can prove that $(\hat{\mathbf{W}}^*,\mathbf{z}^*)$ is a local optimal solution by considering the following two cases. 
\begin{description}
    \item[(Case 1)] If $\text{rank}(\hat{\mathbf{W}}^*)=r$, then $\nabla_{\hat{\mathbf{W}}} \mathscr{L}(\hat{\mathbf{W}}^{*},\mathbf{z}^*,\bm{\lambda}^{*})\in\mathrm{N}_{\mathcal{R}}(\hat{\mathbf{W}}^*)$. For any $\hat{\mathbf{W}}\in\mathcal{F}\cap U(\hat{\mathbf{W}}^*;\delta_1)$ with  $\delta_1>0$, denoting the radius of an open ball neighborhood. 
    By the orthogonality of the normal cone and the tangent cone, it has
    \begin{equation}\label{W22}
\langle\nabla_{\hat{\mathbf{W}}} \mathscr{L}(\hat{\mathbf{W}}^*,\mathbf{z}^*,\bm{\lambda}^*),\hat{\mathbf{W}}-\hat{\mathbf{W}}^*\rangle= 0.
	\end{equation}
    
\item[(Case 2)] If $\text{rank}(\hat{\mathbf{W}}^*)<r$, then $\nabla_{\hat{\mathbf{W}}} \mathscr{L}(\hat{\mathbf{W}}^{*},\mathbf{z}^*,\bm{\lambda}^{*})=0$. For any $\hat{\mathbf{W}}\in\mathbb{R}^{(p+1)\times q}$, it also has \eqref{W22}.
\end{description}
Due to the lower semicontinuous of function $\|{\mathbf{z}}_+\|_0$, for any $\delta_2>0$, there is a neighborhood $U(\mathbf{z}^{*};\delta_{2})$  of $\mathbf{z}^{*}\in\mathcal{L}$  such that
 $$\|{\mathbf{z}}_+\|_0>\|{\mathbf{z}}^*_+\|_0-\dfrac{1}{2}.$$ Since $\|{\mathbf{z}}_+\|_0$ and $\|{\mathbf{z}}^*_+\|_0$ can only take integer values, then it should have 
 \begin{equation}\label{zz}\|{\mathbf{z}}_+\|_0\geq\|{\mathbf{z}}^*_+\|_0 .\end{equation}
As $f(\mathbf{T})$ is the objective function, thus for any $\mathbf{T}\in\mathcal{F}\cap U\left(\mathbf{T}^{*};\min\{\delta_{1},\delta_{2}\}\right)$, it implies
	$$
\begin{aligned}f(\mathbf{T})-f(\mathbf{T}^{*})&=\mathscr{L}(\mathbf{T},\bm{\lambda}^{*})-\mathscr{L}(\mathbf{T}^{*},\bm{\lambda}^{*})\\
&=\frac{1}{2}\langle\hat{\mathbf{W}}-\hat{\mathbf{W}}^*,\hat{\mathbf{W}}-\hat{\mathbf{W}}^*\rangle+\beta\|{\mathbf{z}}_+\|_0-\beta\|{\mathbf{z}}^*_+\|_0\\
&=\mathscr{L}(\hat{\mathbf{W}},\mathbf{z}^*,\bm{\lambda}^{*})-\mathscr{L}(\hat{\mathbf{W}}^{*},\mathbf{z}^*,\bm{\lambda}^{*})+\beta \|{\mathbf{z}}_+\|_0 -\beta\|{\mathbf{z}}^*_+\|_0\\
&\geq\langle\nabla_{\hat{\mathbf{W}}} \mathscr{L}(\hat{\mathbf{W}}^{*},\mathbf{z}^*,\bm{\lambda}^{*}),\hat{\mathbf{W}}-\hat{\mathbf{W}}^{*}\rangle\\&=0,\end{aligned}$$
where the first inequality is due to the convexity of the Lagrangian function with respect to $\hat{\mathbf{W}}$ and \eqref{zz}. Therefore, $(\hat{\mathbf{W}}^*,\mathbf{z}^*)$  is also a local minimizer of problem \eqref{RSMM3}. The proof is completed. 
\end{proof}

\subsection{Optimization Algorithm}

\begin{algorithm}[t]
	\caption{Optimization algorithm for problem \eqref{P}} 
	\label{Alg} 
	\textbf{Input:}  Data $\{(\mathbf{X}_i, y_i)\}_{i=1}^m$, parameters $\beta$, $\sigma$, $r$, $\tau_1$, $\tau_2$, $\tau_3$.\\ 
	\textbf{Output:} $\mathbf{W}$, $\mathbf{z}$, $b$.\\
	\textbf{For} $k = 1$ \textbf{to} $\texttt{maxit}$ \textbf{do}
	\begin{algorithmic}[1]
		 \item  Update $\mathbf{W}^{k+1}$ by solving
		\begin{equation*}
			\begin{aligned}
				\underset{\mathbf{W}}{\min}\quad &f(\mathbf W,{\mathbf{z}}^{k},b^{k})+\frac{\tau_{1}}{2}\|\mathbf{W}-\mathbf{W}^{k}\|_{\textrm{F}}^{2}\\
				\text{s.t.}\quad &\operatorname{rank}(\mathbf{W}) \leq r.
			\end{aligned}
		\end{equation*}
		 \item  Update $\mathbf{z}^{k+1}$ by solving
		\begin{equation*}
			\underset{\mathbf{z}}{\min}\quad f(\mathbf{W}^{k+1},\mathbf{z},b^{k}) + \frac{\tau_{2}}{2}\|\mathbf{z}-\mathbf{z}^{k}\|_{\textrm{F}}^{2}.
		\end{equation*}
		 \item  Update $b^{k+1}$ by solving
		\begin{equation*}
			\underset{b}{\min}\quad f(\mathbf{W}^{k+1},\mathbf{z}^{k+1},b) + \frac{\tau_{3}}{2}\|b-b^{k}\|^2.
		\end{equation*}
	\end{algorithmic}
	\textbf{End for}
\end{algorithm}

This section develops a proximal alternating minimization (PAM) algorithm for solving problem \eqref{RSMM3}. 
Consider the following equivalent optimization problem
\begin{equation}\label{P}
\begin{aligned}
    \min\limits_{\mathbf W, \mathbf{z}, b}\quad
    & f(\mathbf{W},\mathbf{z},b)=\frac{1}{2}\langle\mathbf W,\mathbf W\rangle+\beta\|{\mathbf{z}}_+\|_0+\sigma\|{\mathbf{z}}-{\mathbf{1}}+\mathscr{A}(\mathbf{W})+b\mathbf{y}\|^{2}\\
 {\rm s.t.}\quad & \text{rank}(\mathbf{W}) \leq r,
	\end{aligned}
\end{equation}
where $\sigma>0$ is the penalty parameter. Now we discuss the update rules for each subproblem in Algorithm \ref{Alg}, where $\tau_1, \tau_2, \tau_3>0$ are the parameters to balance the quadratic terms.

\begin{itemize}
    \item The $\mathbf{W}$-subproblem can be reduced to 
\begin{equation}\label{sub-W}
	\begin{aligned}
 \min\limits_{\mathbf W}\quad
&\frac{1}{2}\langle\mathbf{W},\mathbf{W}\rangle+\sigma\|{\mathbf{z}}-{\mathbf{1}}+\mathscr{A}(\mathbf{W})+b\mathbf{y}\|^{2}+\frac{\tau_{1}}{2}\|\mathbf{W}-\mathbf{W}^k\|_\textrm{F}^2\\
{\rm s.t.}\quad & \text{rank}(\mathbf{W}) \leq r.
 	\end{aligned}
\end{equation}
Denote the objection function as $h(\mathbf{W})$, which is a continuous differentiable function.
Thus, problem \eqref{sub-W} yields the following solution
\begin{equation*}	
\mathbf{W}^{k+1}=\Pi_{\mathcal{R}}(\mathbf{W}^k-\alpha_k\nabla h(\mathbf{W}^k)),
\end{equation*}
where $\nabla h(\mathbf{W}^k)=\mathbf{W}^k+2\sigma\sum_{i=1}^m{y}_{i}((\mathbf{z}^k)_i-1+{y}_i\langle \mathbf{W}^k, \mathbf{X}_i \rangle +b^k{y}_i)
\mathbf{X}_i.$ 
    \item The $\mathbf{z}$-subproblem can be simplified to
\begin{equation}\label{sub-z}
    \min\limits_{{\mathbf{z}}}\quad
	\beta\|{\mathbf{z}}_+\|_0+\sigma\|{\mathbf{z}}-{\mathbf{1}}+\mathscr{A}(\mathbf{W})+b\mathbf{y}\|^{2}+\frac{\tau_{2}}{2}\|\mathbf{z}-\mathbf{z}^k\|^2.
\end{equation}
Let $\mathbf{v}={\mathbf{1}}-\mathscr{A}(\mathbf{W})-b\mathbf{y}$. Then, the above problem  is equivalent to 
\begin{equation*}  
\min\limits_{{\mathbf{z}}}\quad \frac{2\beta}{\sigma+\tau_2}\|{\mathbf{z}}_+\|_0+\|{\mathbf{z}}-\frac{2\sigma}{\sigma+\tau_2}\mathbf{v}-\frac{\tau_2}{\sigma+\tau_2}\mathbf{z}^k\|^2.
\end{equation*} 
By exploiting the hard thresholding operator, problem \eqref{sub-z} has the solution
\begin{equation*}
	\mathbf{z}^{k+1}= \mathrm{Prox}_{\frac{2\beta}{\sigma+\tau_2}\|(\cdot)_+ \|_0} (\frac{2\sigma}{\sigma+\tau_2}\mathbf{v}+\frac{\tau_2}{\sigma+\tau_2}\mathbf{z}^k),\nonumber
\end{equation*} 
where 
\begin{equation*}
	(\mathbf{z}^{k+1})_i=\begin{cases}
	0, &\quad 0<	(\frac{2\sigma}{\sigma+\tau_2}\mathbf{v}+\frac{\tau_2}{\sigma+\tau_2}\mathbf{z}^k)_i\leq \sqrt{\frac{4\beta}{\sigma+\tau_2}},\\
	(\frac{2\sigma}{\sigma+\tau_2}\mathbf{v}+\frac{\tau_2}{\sigma+\tau_2}\mathbf{z}^k)_i, &\quad \text{otherwise}.
\end{cases} 
\end{equation*}
    \item The $b$-subproblem is a convex quadratic programming problem
\begin{equation*}
\min\limits_{b}\quad\sigma\|{\mathbf{z}}-{\mathbf{1}}+\mathscr{A}(\mathbf{W})+b\mathbf{y}\|^{2}+\frac{\tau_{3}}{2}\|b-b^k\|^2.
\end{equation*}
It has the first-order derivative
\begin{equation*}
	\begin{aligned}
		2 \sigma \mathbf{\mathbf{y}}^\top ( \mathbf{z}^{k+1} - \mathbf{{\mathbf{1}}} + \mathscr{A}(\mathbf{W}^{k+1}) + b^{k+1} \mathbf{\mathbf{y}} )+\tau_3(b^{k+1}-b^k)=0,
	\end{aligned}
\end{equation*}
and thus
$$b^{k+1}=\\
\frac{\tau_3b^k-2\sigma\mathbf{y}^\top({\mathbf{z}}^{k+1}-{\mathbf{1}}+\mathscr{A}(\mathbf W^{k+1}))}{2\sigma\mathbf{y}^\top\mathbf{y}+\tau_3}.$$
\end{itemize}

\subsection{Convergence Analysis}
This section provides convergence guarantees for Algorithm \ref{Alg}.  Let $\Phi(\mathbf{W})$ be the  indicator function defined as
\begin{equation*}
	\Phi(\mathbf{W})=\begin{cases}
		0,  & \mathbf{W}\in\mathcal{R} \\
		+\infty, & \text{otherwise},
	\end{cases}
\end{equation*}
and
\begin{equation*}
g(\mathbf{W},\mathbf{z},b)=\frac{1}{2}\langle\mathbf W,\mathbf W\rangle+\beta\|{\mathbf{z}}_+\|_0+\sigma\|{\mathbf{z}}-{\mathbf{1}}+\mathscr{A}(\mathbf{W})+b\mathbf{y}\|^{2}+\Phi(\mathbf{W}).
\end{equation*}

 \begin{theorem}\label{C2}
	For the sequence $\{(\mathbf{W}^k,\mathbf{z}^k,b^k)\}$ obtained by Algorithm \ref{Alg}, there exists $\tau>0$ that satisfies
\begin{equation}\label{ccc}
    \begin{aligned}
    &g(\mathbf{W}^k,\mathbf{z}^k,b^k)-g(\mathbf{W}^{k+1},\mathbf{z}^{k+1},b^{k+1}) \\
    &\geq \frac{\tau}{2}(\|\mathbf{W}^{k+1}-\mathbf{W}^k\|_{\textrm{F}}^2 + \|\mathbf{z}^{k+1}-\mathbf{z}^k\|^2 + \|b^{k+1}-b^k\|^2).
    \end{aligned}
\end{equation}
\end{theorem}
\begin{proof}
From the updates of Algorithm \ref{Alg}, it is seen that 
%Let $(\mathbf{W}^{k+1},\mathbf{z}^{k+1},b^{k+1})$  be the optimal solution, it has
\begin{equation*}
    \begin{aligned}
    &g(\mathbf{W}^{k+1},\mathbf{z}^k,b^k)+\frac{\tau_1}{2}\|\mathbf{W}^{k+1}-\mathbf{W}^k\|_{\textrm{F}} ^2\leq g(\mathbf{W}^{k},\mathbf{z}^k,b^k),\\
    &g(\mathbf{W}^{k+1},\mathbf{z}^{k+1},b^k)+\frac{\tau_2}{2}\|\mathbf{z}^{k+1}-\mathbf{z}^k\|^2\leq g(\mathbf{W}^{k+1},\mathbf{z}^k,b^k),\\
    &g(\mathbf{W}^{k+1},\mathbf{z}^{k+1},b^{k+1})+\frac{\tau_3}{2}\|b^{k+1}-b^k\| ^2\leq g(\mathbf{W}^{k+1},\mathbf{z}^{k+1},b^k).\\
    \end{aligned}
\end{equation*}
The above inequalities, together with the fact that there exists $\tau=\max\{\tau_1,\tau_2,\tau_3\}$, make \eqref{ccc}  naturally hold. Therefore, the proof is completed.
\end{proof}

\begin{remark}
Note that the convergence analysis of PAM relies on the Kurdyka-Łojasiewicz (KL) property of the objective function  \cite{bolte2014proximal}. Although \cite{schneider2015convergence2} demonstrates that the KL property is attainable when the objective function contains a continuously differentiable component over a low-rank set, our formulation \eqref{P}  presents a distinct theoretical challenge due to the Heaviside loss. The resulting jump discontinuities and vanishing subdifferentials almost everywhere preclude the existence of a desingularizing function, thereby violating the KL inequality. Consequently, as the standard descent-based convergence cannot be theoretically derived, we provide numerical evidence in Section \ref{sec4} to empirically validate the robust convergence and stability of the proposed algorithm.
\end{remark}

\begin{table*}[t]
\caption{Summary of the selected six datasets.} \label{datatab}
\centering
\renewcommand{\arraystretch}{1.12}
\begin{adjustbox}{width=0.8\textwidth}
\begin{tabular}{lccc}
\toprule
Datasets & Dimensions & Training Samples (\#pos/\#neg) & Testing Samples (\#pos/\#neg) \\
\midrule
SPAMBASE     & $4601\times57$                 & 2255/966 & 966/414 \\
IONO   & $351\times34$                 & 172/74    & 74/31   \\
CIFAR10    & $32\times32$                   & 3500/1500 & 700/300  \\
CaltechFace & $896\times592$                 & 221/94       & 94/41     \\
BCI  & $200\times59$                  & 98/42     & 42/18   \\
WDBC  & $596\times30$                  & 279/119   & 120/51  \\ \bottomrule
\end{tabular}
\end{adjustbox}
\end{table*}

To end this section, the computational complexity of our proposed PAM algorithm is analyzed.
Updating $\mathbf{W}$ involves calculating the gradient $\nabla h(\mathbf{W})$ and projecting the result onto the rank constraint set $\mathcal{R}$, which require  $\mathcal{O}(mpq)$ and $\mathcal{O}(pqr)$, respectively.
Updating $\mathbf{z}$ and $b$ require $\mathcal{O}(m)$ and $\mathcal{O}(mpq)$, respectively.
Therefore, the overall computational complexity of each iteration is $\mathcal{O}(mpq + pqr)$.

\section{Numerical Experiments}\label{sec4}
In this section, numerical experiments are conducted to evaluate the performance of our proposed HL-SMM on real-world datasets with state-of-the-art SMM methods such as Hinge-SMM \cite{luo2015support}, Pinball-SMM \cite{feng2022support}, Ramp-SMM \cite{gu2021ramp}, and LS-SMM \cite{liang2022adaptive}, as well as other SVM methods including Linear-SVM \cite{cortes1995support}, RBF-SVM \cite{boser1992training}, and Poly-SVM \cite{shen2025bafl}. All experiments are performed using MATLAB 2021b on an Intel i5-1135G7 CPU at 2.40 GHz with 16 GB RAM. Our code will be available at \url{https://github.com/xianchaoxiu/HL-SMM}.

\subsection{Experimental Setup}

In the experiments, six real-world datasets are chosen, including 
SPAMBASE\footnote{\url{https://archive.ics.uci.edu/dataset/94/spambase}}, 
IONO\footnote{\url{https://archive.ics.uci.edu/dataset/52/ionosphere}}, 
CIFAR10\footnote{\url{https://www.cs.toronto.edu/~kriz/cifar.html}}, 
CaltechFace\footnote{\url{https://www.vision.caltech.edu/Image_Datasets/faces/}},
BCI\footnote{\url{https://www.bbci.de/competition/iv/}}, 
and
WDBC\footnote{\url{https://archive.ics.uci.edu/dataset/17/breast+cancer+wisconsin+diagnostic}},
see Table  \ref{datatab} for more details.
For the CIFAR10 dataset, only the first two classes are used to construct binary classification tasks. 
For the BCI dataset, the EEG signals are segmented into 4-second segments, each starting 0.5 seconds after the cue and downsampled by a factor of 20. Segments exceeding the recording boundaries are discarded, and shorter segments are zero-padded. 
For all image datasets, each sample is converted to grayscale and normalized to zero mean and unit variance.

\begin{table*}[t]
\caption{Classification accuracy (\%) comparison of different methods.}
\label{table1}
\centering
\setlength{\tabcolsep}{4pt} 
\renewcommand{\arraystretch}{1.12} 
\begin{adjustbox}{width=\textwidth}
\begin{tabular}{lccccccccc}
\toprule 
Datasets 
& Hinge-SMM 
& Pinball-SMM 
& Ramp-SMM 
& LS-SMM 
& Linear-SVM 
& RBF-SVM 
& Poly-SVM
& HL-SMM \\
\midrule  
SPAMBASE 
& 88.97 & 87.83 & 88.71 & 86.30 & 87.75 & 71.16 & 90.43 & \textbf{90.96} \\
IONO 
& 85.90 & 82.10 & 71.05 & 80.00 & \textbf{86.71} & 64.67 & 84.76 & 86.48 \\
CIFAR10 
& 74.40 & 74.25 & 70.35 & 47.50 & 67.75 & 61.00 & 74.75 & \textbf{75.55} \\
CaltechFace
& 96.50 & 96.13 & 94.75 & 91.24 & \textbf{100.00} & 60.37 & 96.31 & 99.08 \\
BCI 
& 50.67 & 51.00 & 49.67 & 45.00 & 45.00 & 46.67 & 48.33 & \textbf{56.00} \\
WDBC 
& 94.15 & 95.09 & 92.05 & 84.80 & 94.15 & 97.08 & 97.83 & \textbf{98.25} \\
\midrule  
Average
& 81.77 & 81.07 & 77.76 & 72.47 & 80.23 & 66.83 & 82.07 & \textbf{84.39} \\
\bottomrule  
\end{tabular}
\end{adjustbox}
\end{table*}

\begin{table*}[t]
\caption{Classification accuracy (\%) under different Gaussian noise levels.}
\label{table2}
\centering
\setlength{\tabcolsep}{3pt}  
\renewcommand{\arraystretch}{1.12}  
\small  

\begin{adjustbox}{width=\textwidth}
\begin{tabular}{lccccccccc}
\toprule  
Datasets & Noise
& Hinge-SMM
& Pinball-SMM
& Ramp-SMM
& LS-SMM
& Linear-SVM
& RBF-SVM
& Poly-SVM
& HL-SMM \\
\midrule 

\multirow{5}{*}{SPAMBASE}
& 0.00 & 88.97 & 88.07 & 89.61 & 86.30 & 87.75 & 71.16 & 90.43 & \textbf{90.96} \\
& 0.05 & 89.19 & 88.03 & 89.00 & 86.09 & 87.51 & 70.25 & 89.91 & \textbf{90.86} \\
& 0.10 & 88.86 & 88.12 & 87.00 & 86.12 & 87.58 & 65.07 & 89.83 & \textbf{90.74} \\
& 0.15 & 88.88 & 88.09 & 88.30 & 86.09 & 87.10 & 60.68 & 88.64 & \textbf{90.49} \\
& 0.20 & 88.61 & 88.17 & 87.35 & 85.33 & 87.29 & 60.58 & 88.64 & \textbf{90.55} \\
\midrule  

\multirow{5}{*}{IONO}
& 0.00 & 86.10 & 82.10 & 74.86 & 80.00 & \textbf{86.67} & 64.76 & 84.76 & 86.48 \\
& 0.05 & 85.90 & 82.10 & 81.52 & 80.57 & 86.48 & 63.81 & \textbf{86.86} & 86.67 \\
& 0.10 & 85.90 & 82.29 & 78.86 & 80.76 & 85.90 & 62.67 & \textbf{87.62} & 85.71 \\
& 0.15 & 86.29 & 81.33 & 68.76 & 82.29 & 86.48 & 61.90 & \textbf{86.86} & 85.71 \\
& 0.20 & \textbf{86.67} & 81.71 & 71.62 & 82.29 & 85.33 & 61.90 & 84.38 & 85.52 \\
\midrule  

\multirow{5}{*}{CIFAR10}
& 0.00 & 74.40 & 74.25 & 70.00 & 47.50 & 67.75 & 51.00 & 74.75 & \textbf{75.55} \\
& 0.05 & 73.90 & 74.05 & 69.90 & 48.80 & 67.80 & 51.00 & \textbf{75.10} & 74.75 \\
& 0.10 & 74.55 & 74.20 & 69.45 & 48.95 & 67.85 & 50.90 & 74.30 & \textbf{74.60} \\
& 0.15 & 73.65 & 74.40 & 69.65 & 54.45 & 67.30 & 50.95 & \textbf{74.90} & 74.70 \\
& 0.20 & 74.85 & 73.55 & 70.40 & 47.50 & 66.15 & 50.90 & 74.50 & \textbf{75.20} \\
\midrule  

\multirow{5}{*}{CaltechFace}
& 0.00 & 95.67 & 96.41 & 95.39 & 91.24 & \textbf{100.00} & 60.37 & 96.31 & 98.71 \\
& 0.05 & 96.04 & 95.21 & 95.48 & 91.24 & \textbf{100.00} & 60.37 & 95.48 & 98.53 \\
& 0.10 & 95.02 & 95.39 & 92.35 & 91.34 & \textbf{100.00} & 60.37 & 95.02 & 98.53 \\
& 0.15 & 96.13 & 95.85 & 87.10 & 91.43 & \textbf{100.00} & 60.37 & 95.02 & 98.25 \\
& 0.20 & 95.85 & 96.13 & 93.46 & 90.32 & \textbf{100.00} & 60.37 & 94.65 & 98.34 \\
\midrule  

\multirow{5}{*}{BCI}
& 0.00 & 50.67 & 51.00 & 49.67 & 45.00 & 45.00 & 46.67 & 48.33 & \textbf{56.00} \\
& 0.05 & 52.00 & 50.33 & 49.00 & 45.67 & 45.67 & 46.33 & 49.33 & \textbf{53.33} \\
& 0.10 & 52.00 & 50.33 & 50.33 & 45.67 & 46.00 & 47.00 & 48.33 & \textbf{53.33} \\
& 0.15 & 51.67 & 51.00 & 50.67 & 46.67 & 45.00 & 48.00 & 48.33 & \textbf{54.33} \\
& 0.20 & \textbf{52.33} & 49.33 & 48.67 & 48.33 & 49.00 & 49.00 & 48.33 & 50.00 \\
\midrule  

\multirow{5}{*}{WDBC}
& 0.00 & 94.15 & 95.09 & 92.05 & 84.80 & 94.15 & 97.08 & 96.49 & \textbf{98.25} \\
& 0.05 & 94.15 & 95.20 & 91.93 & 83.04 & 94.15 & 96.96 & 97.31 & \textbf{97.89} \\
& 0.10 & 94.04 & 95.20 & 92.16 & 91.46 & 94.15 & 97.31 & 97.54 & \textbf{98.25} \\
& 0.15 & 94.15 & 95.56 & 74.85 & 93.92 & 94.27 & 96.96 & 96.84 & \textbf{98.13} \\
& 0.20 & 93.92 & 95.56 & 93.57 & 91.58 & 94.15 & 96.84 & 97.08 & \textbf{97.31} \\
\midrule  

Average & / & 81.82 & 80.94 & 77.43 & 73.83 & 80.22 & 64.05 & 81.86 & \textbf{83.59} \\
\bottomrule  
\end{tabular}
\end{adjustbox}
\end{table*}
%%%%%%%%%%%%%%%%%%

\begin{table*}[t]
\caption{Classification accuracy (\%) under different salt-and-pepper noise levels.}
\label{table3}
\centering
\setlength{\tabcolsep}{3pt}  
\renewcommand{\arraystretch}{1.12}  
\small  

\begin{adjustbox}{width=\textwidth}
\begin{tabular}{lccccccccc}
\toprule  
Datasets & Noise
& Hinge-SMM
& Pinball-SMM
& Ramp-SMM
& LS-SMM
& Linear-SVM
& RBF-SVM
& Poly-SVM
& HL-SMM \\
\midrule  

\multirow{5}{*}{SPAMBASE}
& 0.00 & 88.97 & 88.23 & 89.45 & 86.30 & 87.75 & 71.16 & 90.43 & \textbf{90.96} \\
& 0.05 & 87.96 & 87.80 & 86.42 & 84.68 & 86.77 & 60.67 & 86.03 & \textbf{89.29} \\
& 0.10 & 87.00 & 86.97 & 85.28 & 84.52 & 85.78 & 60.58 & 84.00 & \textbf{87.87} \\
& 0.15 & 85.97 & 85.28 & 83.45 & 83.93 & 84.09 & 60.58 & 81.54 & \textbf{86.17} \\
& 0.20 & 84.83 & 84.52 & 82.72 & 82.51 & 83.45 & 60.58 & 79.19 & \textbf{85.77} \\
\midrule  

\multirow{5}{*}{IONO}
& 0.00 & 86.29 & 83.24 & 71.05 & 80.00 & \textbf{86.67} & 64.76 & 84.76 & 86.48 \\
& 0.05 & \textbf{85.71} & 80.76 & 70.10 & 82.29 & 85.14 & 62.10 & 83.62 & 84.95 \\
& 0.10 & \textbf{84.00} & 79.81 & 73.90 & 82.67 & 83.62 & 61.90 & 79.43 & 83.24 \\
& 0.15 & \textbf{83.24} & 77.33 & 70.67 & 82.29 & 82.29 & 61.90 & 80.95 & 81.52 \\
& 0.20 & \textbf{82.67} & 74.86 & 71.43 & 80.00 & 80.00 & 61.90 & 78.10 & \textbf{82.67} \\
\midrule  

\multirow{5}{*}{CIFAR10}
& 0.00 & 74.50 & 73.75 & 70.65 & 47.50 & 67.75 & 51.00 & 74.75 & \textbf{75.55} \\
& 0.05 & \textbf{74.65} & 73.80 & 70.75 & 52.30 & 69.35 & 50.65 & 73.30 & 74.50 \\
& 0.10 & 73.70 & 74.40 & 70.25 & 49.10 & 66.70 & 50.55 & 71.75 & \textbf{74.50} \\
& 0.15 & 73.60 & \textbf{74.30} & 69.65 & 50.40 & 68.65 & 50.00 & 70.85 & 73.60 \\
& 0.20 & 73.25 & 72.95 & 68.75 & 51.70 & 68.70 & 50.00 & 70.35 & \textbf{73.95} \\
\midrule  

\multirow{5}{*}{CaltechFace}
& 0.00 & 97.14 & 95.58 & 87.10 & 91.24 & \textbf{100.00} & 60.37 & 96.31 & 98.71 \\
& 0.05 & 97.33 & 96.77 & 91.61 & 90.69 & \textbf{100.00} & 60.37 & 95.21 & 98.06 \\
& 0.10 & 96.31 & 94.10 & 93.82 & 91.34 & \textbf{100.00} & 60.37 & 94.75 & 98.06 \\
& 0.15 & 96.13 & 95.76 & 95.58 & 89.59 & \textbf{100.00} & 60.37 & 96.31 & 97.70 \\
& 0.20 & 95.67 & 95.30 & 93.64 & 87.83 & \textbf{100.00} & 60.37 & 91.98 & 97.70 \\
\midrule  

\multirow{5}{*}{BCI}
& 0.00 & 51.33 & 49.33 & 48.33 & 45.00 & 45.00 & 46.67 & 48.33 & \textbf{56.00} \\
& 0.05 & 50.33 & 50.33 & 50.00 & 50.33 & 49.33 & 50.67 & 48.67 & \textbf{53.33} \\
& 0.10 & 53.67 & 51.33 & 49.67 & \textbf{55.00} & \textbf{55.00} & 51.67 & 49.00 & 52.33 \\
& 0.15 & 53.67 & 48.67 & 48.00 & 56.33 & \textbf{57.33} & 51.00 & 49.33 & 52.33 \\
& 0.20 & 54.00 & 49.00 & 49.00 & 56.67 & \textbf{59.33} & 50.00 & 48.33 & 56.00 \\
\midrule  

\multirow{5}{*}{WDBC}
& 0.00 & 94.15 & 95.32 & 92.16 & 84.80 & 94.15 & 97.08 & 96.49 & \textbf{98.25} \\
& 0.05 & 93.68 & 95.79 & 92.28 & 93.33 & 93.80 & 96.61 & 96.14 & \textbf{97.78} \\
& 0.10 & 93.80 & 95.20 & 92.63 & 94.85 & 92.98 & 96.73 & 96.37 & \textbf{96.96} \\
& 0.15 & 92.98 & 95.91 & 92.05 & 94.04 & 92.40 & \textbf{96.96} & 95.91 & 95.20 \\
& 0.20 & 93.33 & 95.91 & 91.11 & 94.39 & 91.46 & 96.02 & \textbf{96.14} & 96.02 \\
\midrule  

Average & / & 81.33 & 80.08 & 76.72 & 75.19 & 80.58 & 63.79 & 79.61 & \textbf{82.52} \\
\bottomrule  
\end{tabular}
\end{adjustbox}
\end{table*}

\begin{figure*}[t]
    \centering
    \begin{subfigure}[b]{0.30\textwidth}
        \centering
        \includegraphics[width=\linewidth]{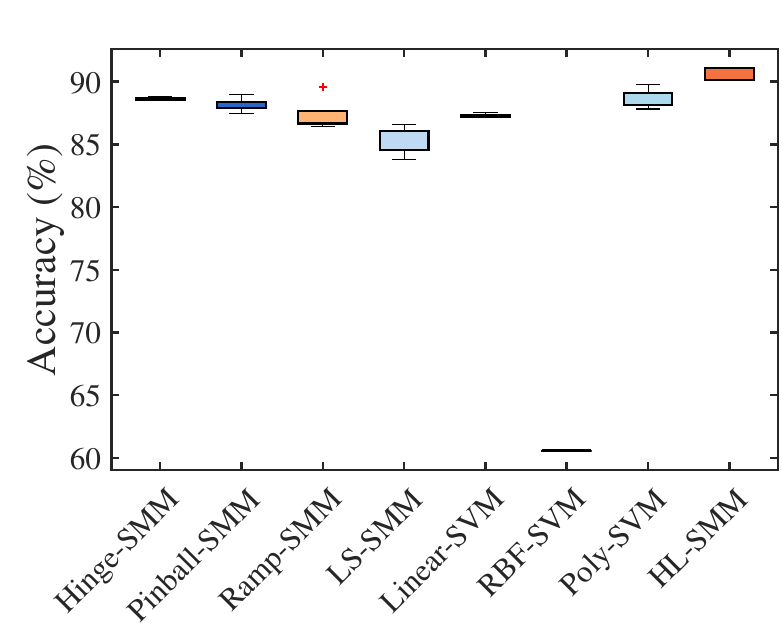}
        \caption{SPAMBASE}
    \end{subfigure}
     \hfill    
       \begin{subfigure}[b]{0.30\textwidth}
        \centering
        \includegraphics[width=\linewidth]{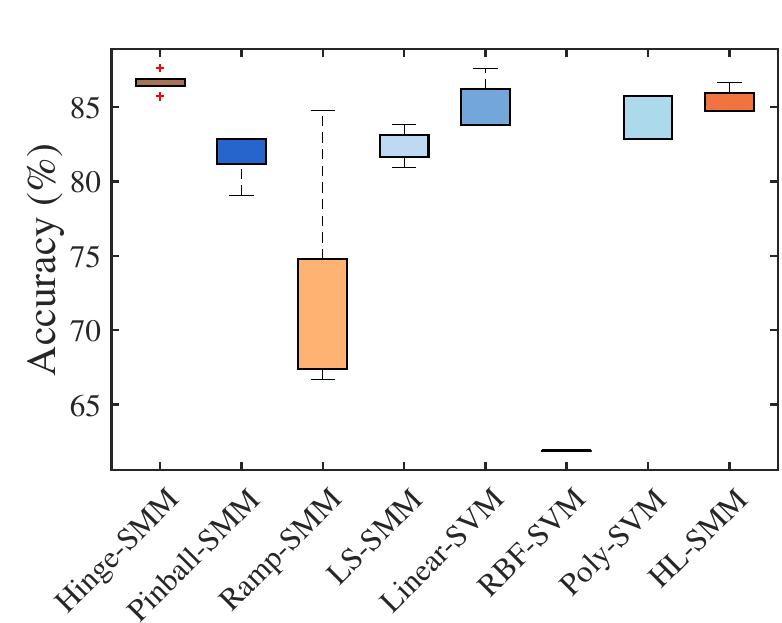}
        \caption{IONO}
    \end{subfigure}
    \hfill 
       \begin{subfigure}[b]{0.30\textwidth}
        \centering
        \includegraphics[width=\linewidth]{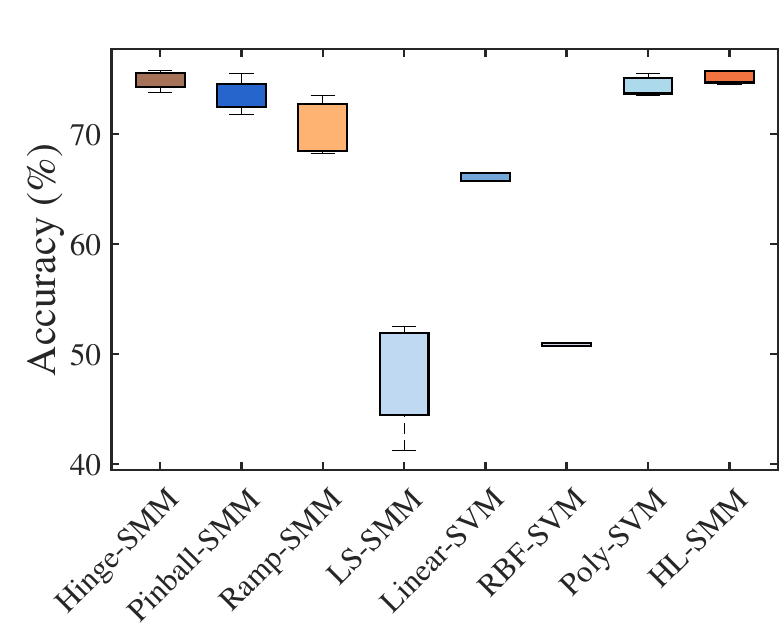}
        \caption{CIFAR10}
    \end{subfigure}
    
  %      \vspace{0.15cm} 
        
        \begin{subfigure}[b]{0.30\textwidth}
        \centering
        \includegraphics[width=\linewidth]{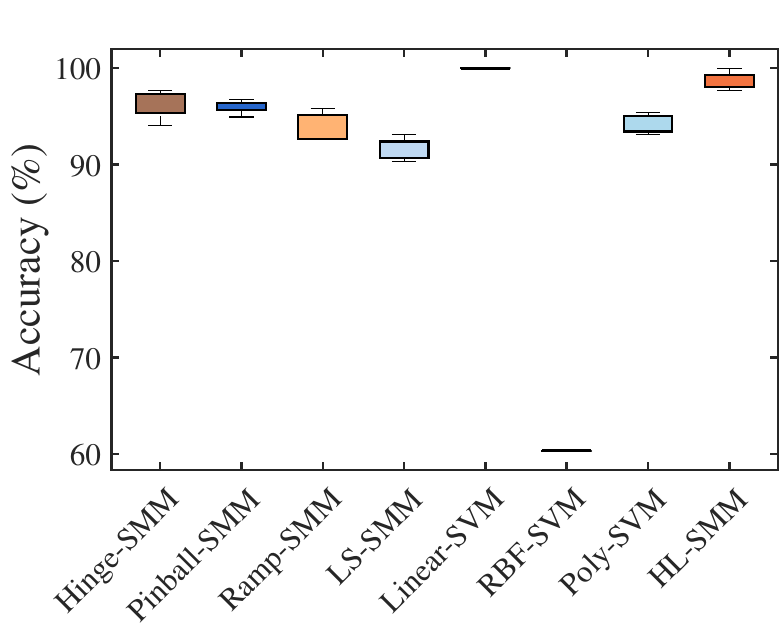}
        \caption{CaltechFace}
    \end{subfigure}     
        \hfill
    \begin{subfigure}[b]{0.30\textwidth}
        \centering
        \includegraphics[width=\linewidth]{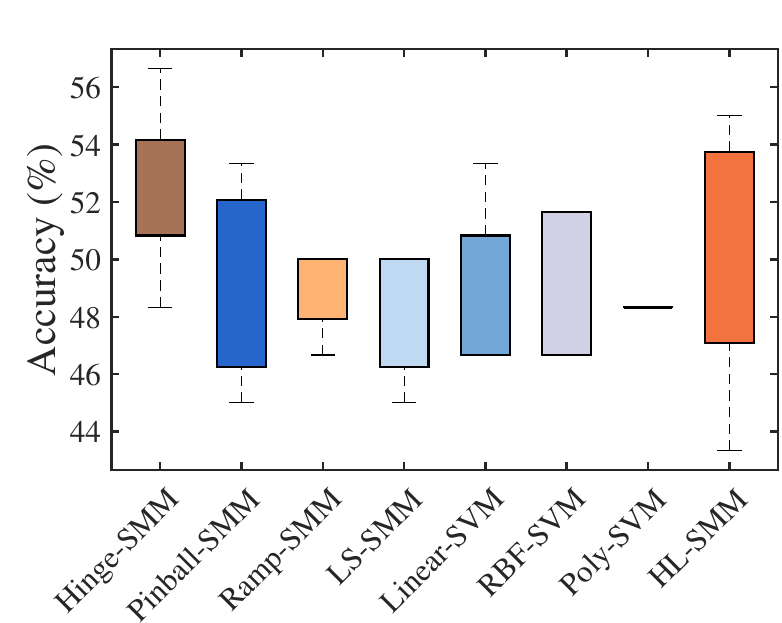}
        \caption{BCI}
    \end{subfigure}
        \hfill 
    \begin{subfigure}[b]{0.30\textwidth}
        \centering
        \includegraphics[width=\linewidth]{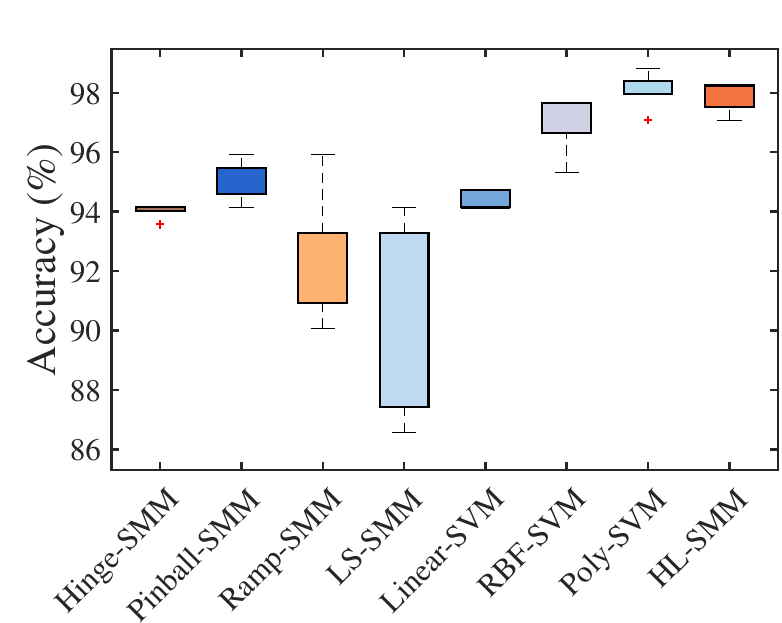}
        \caption{WDBC}
    \end{subfigure}

    \caption{Robustness comparison under 20\% Gaussian noise.}
    \label{fig2}
\end{figure*}

\begin{figure*}[t]
    \centering
    \begin{subfigure}[b]{0.30\textwidth}
        \centering
        \includegraphics[width=\linewidth]{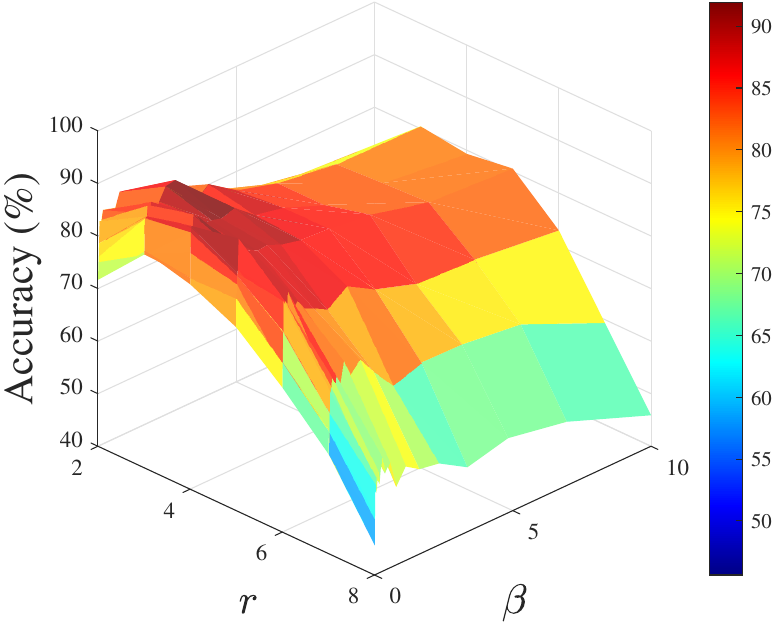}
        \caption{SPAMBASE} 
    \end{subfigure}
      \hfill  
      \begin{subfigure}[b]{0.30\textwidth}
        \centering
        \includegraphics[width=\linewidth]{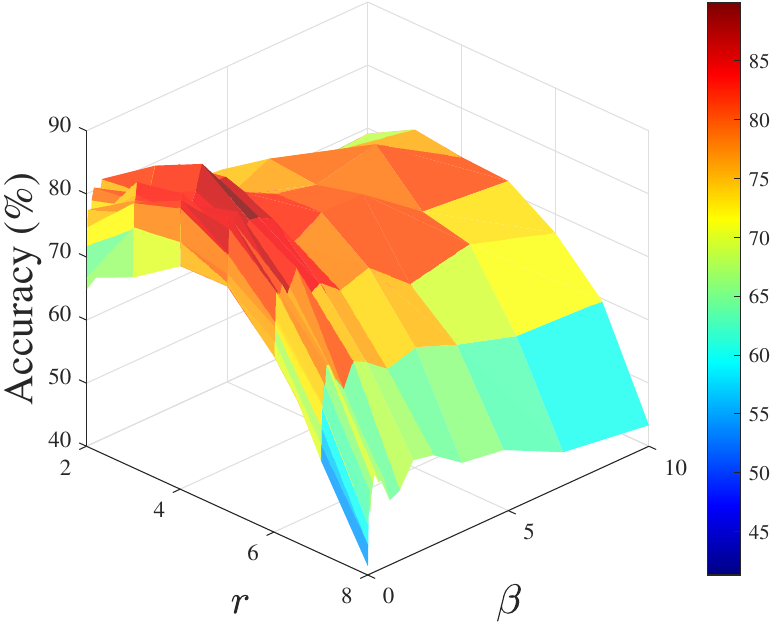}
        \caption{IONO} 
    \end{subfigure}  
      \hfill
       \begin{subfigure}[b]{0.30\textwidth}
        \centering
        \includegraphics[width=\linewidth]{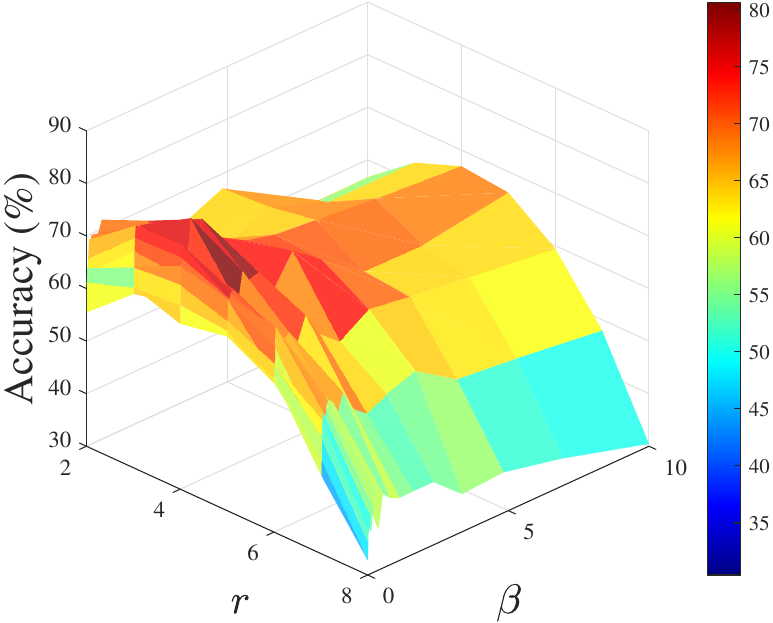}
        \caption{CIFAR10} %
    \end{subfigure}   
       % \vspace{0.15cm}   
        \begin{subfigure}[b]{0.30\textwidth}
        \centering
        \includegraphics[width=\linewidth]{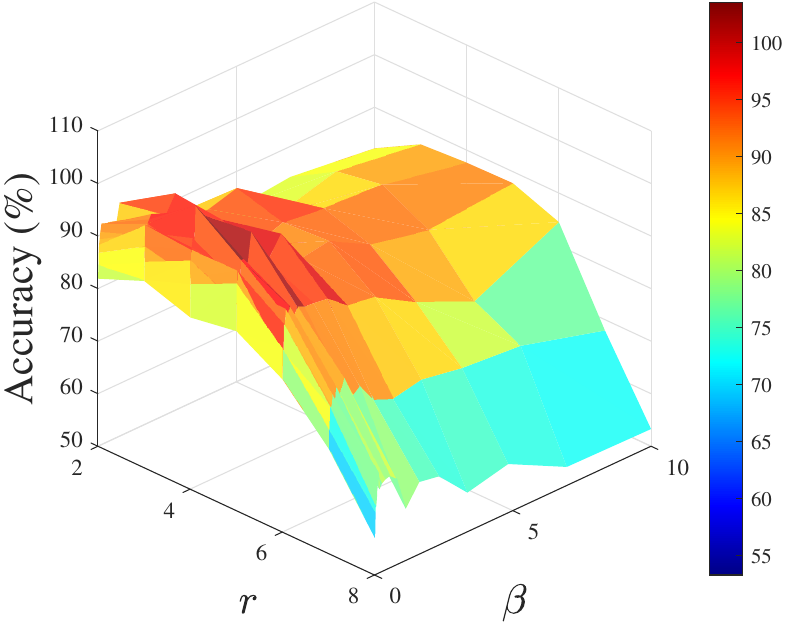}
        \caption{CaltechFace} 
    \end{subfigure}  
         \hfill   
    \begin{subfigure}[b]{0.30\textwidth}
        \centering
        \includegraphics[width=\linewidth]{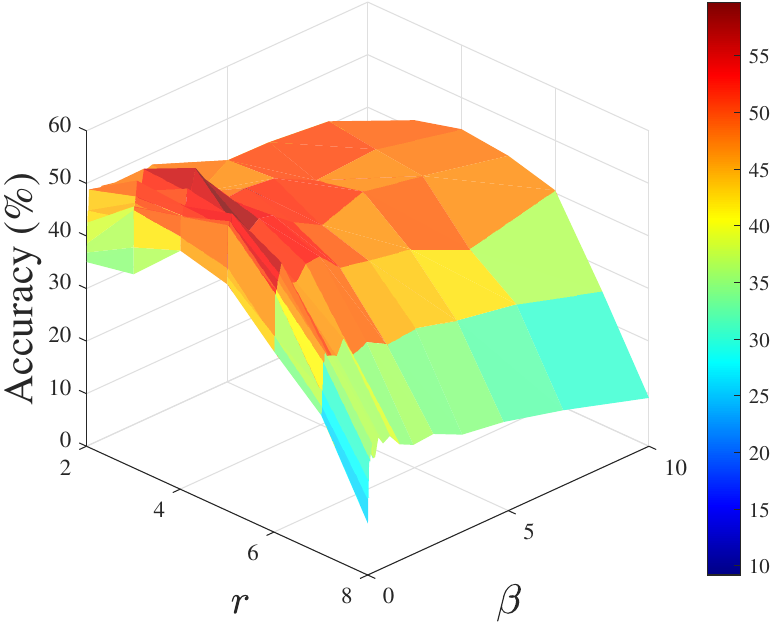}
        \caption{BCI} 
    \end{subfigure}
    \hfill
    \begin{subfigure}[b]{0.30\textwidth}
        \centering
        \includegraphics[width=\linewidth]{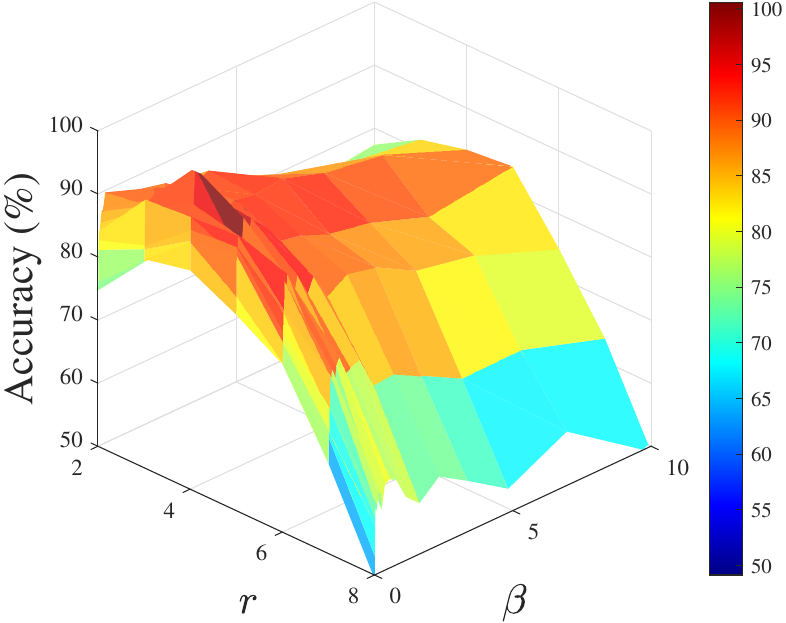}
        \caption{WDBC} 
    \end{subfigure}

    \caption{Parameter sensitivity analysis of $r$ and $\beta$.}
    \label{fig3}
\end{figure*}

The hyperparameters are optimized on the following candidate sets: $\beta \in \{0.01, 0.1, 0.5\}$, $\sigma \in \{0.01,0.1\}$, $r \in \{4,10\}$, $\tau_1, \tau_2, \tau_3\in \{10^{-4}, 10^{-3}, 10^{-2}\}$, the maximum iteration $\texttt{maxit} = 1000$. Note that $\beta$ and $r$ are the two key hyperparameters that balance margin maximization and classification error. A detailed sensitivity analysis  will be presented in subsequent sections. For the remaining baseline methods, all hyperparameters are fixed to the values reported in their original papers.

To evaluate the classification performance, it employs accuracy as the primary metric, which is defined as
\begin{equation*}
\text{Accuracy} = \frac{\text{TP + TN}}{\text{TP + TN + FP + FN}} \times 100\%,
\end{equation*}
where $\text{TP}, \text{TN}, \text{FP}$, and $\text{FN}$ represent true positives, true negatives, false positives, and false negatives, respectively. A higher classification accuracy indicates a better classification result.
 %The final results are listed in Table~\ref{table1}.
 
\subsection{Experimental Results}

Table \ref{table1} shows the classification accuracy of different methods on six datasets, with the best results highlighted in bold. It can be observed that our proposed HL-SMM achieves the best classification accuracy on average, improving upon the second-best by 2.32\%.
In particular, on the BCI dataset, HL-SMM significantly outperforms the best competing baseline methods, demonstrating the effectiveness of the proposed robust low-rank formulation in challenging scenarios.
On the SPAMBASE and CIFAR10 datasets, HL-SMM also yields consistent improvements over the SMM variants with convex or nonconvex losses, indicating that the Heaviside loss function can better enhance the performance of SMM.
On the IONO and CaltechFace datasets, although Linear-SVM achieves the highest accuracy, HL-SMM still provides competitive performance, and significantly improves upon LS-SMM and other SMM baselines.

To further evaluate robustness, Table \ref{table2} and Table \ref{table3} report classification results by injecting different Gaussian noise and salt-and-pepper noise, respectively.
It can be concluded that as the noise level increases, most baseline methods exhibit noticeable performance degradation, especially vector-based SVM methods and LS-SMM.
In contrast, our proposed HL-SMM maintains relatively stable accuracy across different noise levels on multiple datasets.
For example, on the SPAMBASE dataset, HL-SMM achieves accuracy above 90\% for all noise levels, while RBF-SVM drops substantially under noise corruption.
Similar robustness trends can also be seen on the WDBC and CaltechFace datasets, indicating that the proposed Heaviside loss can effectively suppress the influence of outliers, and the explicit rank constraint helps preserve the intrinsic low-dimensional structure of matrix data.
Overall, these results validate that our proposed method is not only competitive in clean settings but also more robust in the presence of noise.

Fig. \ref{fig2} presents boxplots of classification accuracy under 20\% Gaussian noise.
Generally speaking, our proposed HL-SMM obtains comparable median accuracy with a tighter interquartile range across datasets, indicating improved robustness and stability under severe corruption.
In particular, on the SPAMBASE and CIFAR10 datasets, HL-SMM shows high and concentrated accuracy distributions, whereas several baselines exhibit noticeable performance degradation and larger variability.
These results confirm that our proposed method generalizes more reliably than existing methods in high-noise scenarios.

\subsection{Discussions}
\subsubsection{Parameter  Analysis}\label{para_sen}
Fig. \ref{fig3} depicts the accuracy landscapes of our proposed HL-SMM over the rank constraint $r$ and the regularization parameter $\beta$ on six datasets.
A clear high-performance plateau can be observed on all datasets, indicating that HL-SMM is not overly sensitive to hyperparameter tuning and yields consistently strong performance within a broad region of $(r,\beta)$.
It can be found that the best results are attained at moderate ranks, whereas too small $r$ tends to restrict model capacity and degrade accuracy, and overly large $r$ may weaken the low-rank inductive bias and reduce robustness.
Regarding $\beta$, the performance typically peaks at intermediate values, while extreme choices may lead to either insufficient margin enforcement or over-penalization.
Therefore, the explicit rank constraint provides an effective mechanism for controlling model complexity and contributes to stable generalization across different datasets.

\subsubsection{Convergence Visualization}
Fig. \ref{fig4} reports the convergence behavior of our proposed PAM algorithm on six  datasets.
The evolution of the iterative loss exhibits a rapid transient phase, followed by a clear tendency to saturate.
This trend indicates that the norm $\|W^{k+1} - W^k\|_\textrm{F}$ decays with the number of iterations of Algorithm \ref{Alg}. 
This result, to some extent, proves that our proposed method can converge to the optimal stationary point.

\begin{figure*}[t]
    \centering
    \begin{subfigure}[b]{0.30\textwidth}
        \centering
        \includegraphics[width=\linewidth]{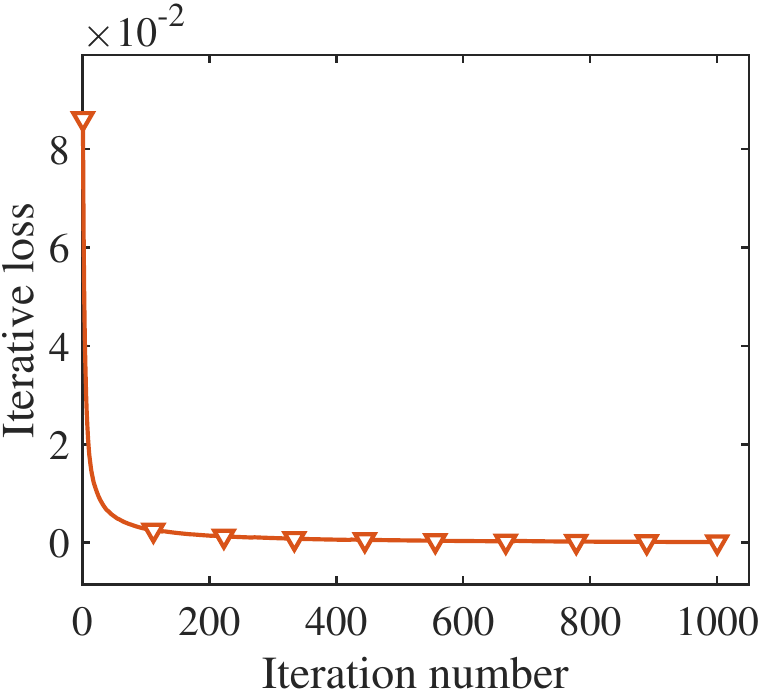}
        \caption{SPAMBASE} 
    \end{subfigure}
     \hfill   
      \begin{subfigure}[b]{0.30\textwidth}
        \centering
        \includegraphics[width=\linewidth]{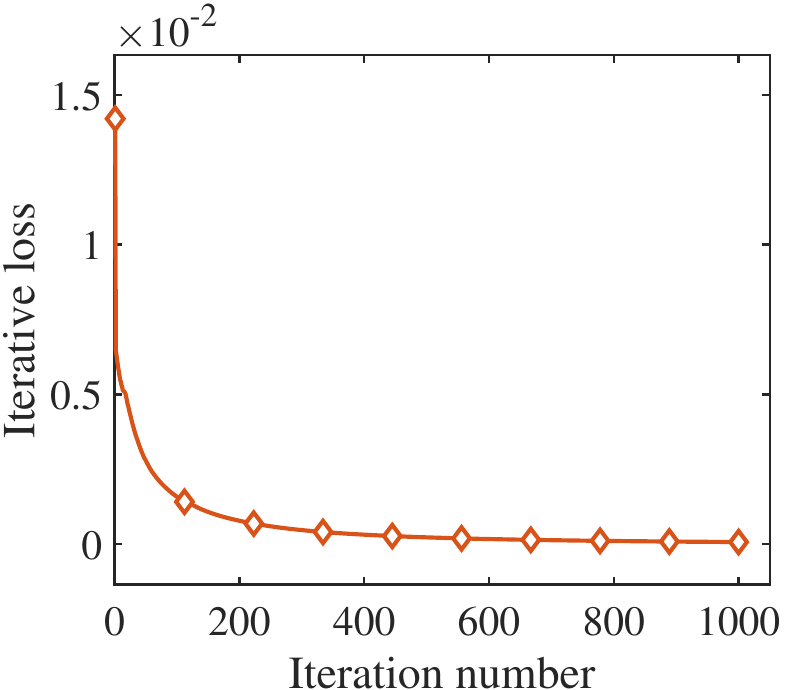}
        \caption{IONO} 
    \end{subfigure}
    \hfill  
      \begin{subfigure}[b]{0.30\textwidth}
        \centering
        \includegraphics[width=\linewidth]{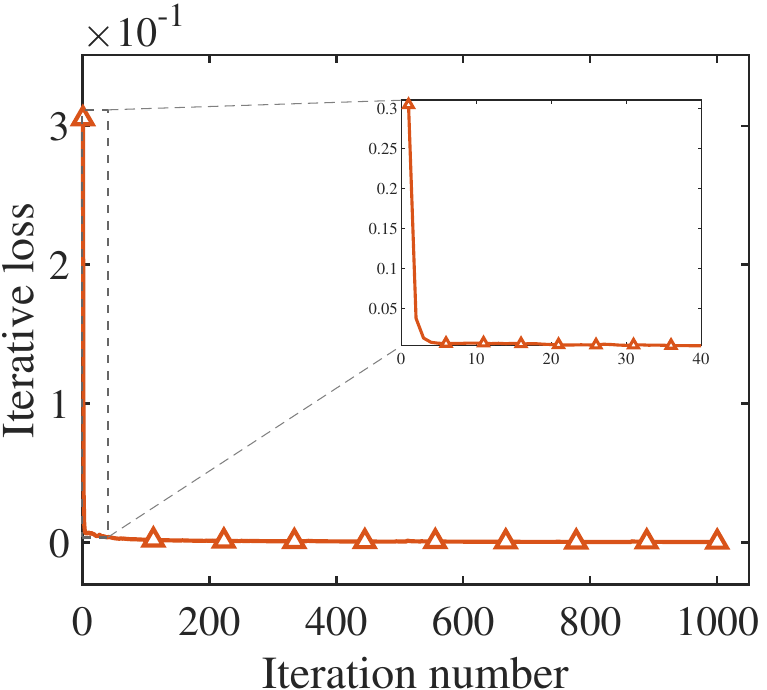}
        \caption{CIFAR10} 
    \end{subfigure}

    \begin{subfigure}[b]{0.30\textwidth}
        \centering
        \includegraphics[width=\linewidth]{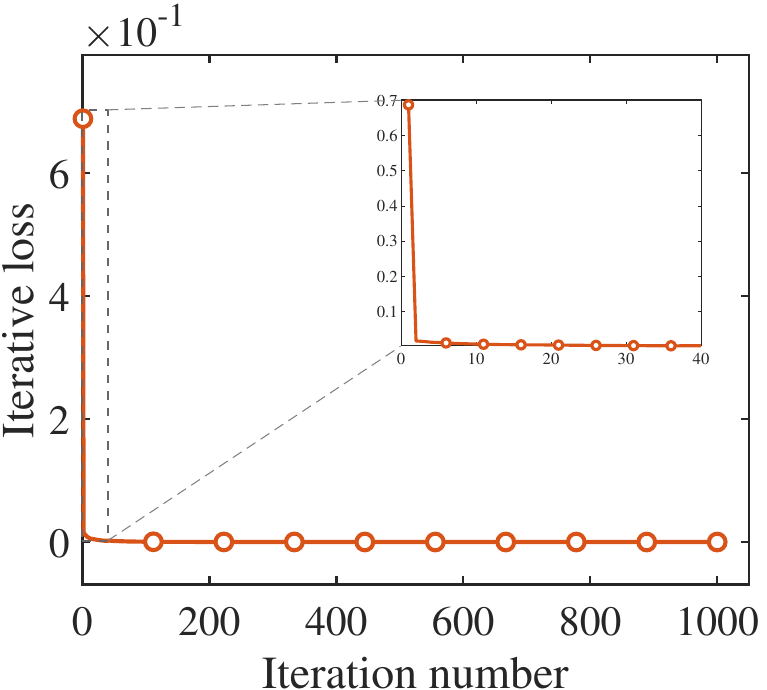}
        \caption{CaltechFace} 
    \end{subfigure}
    \hfill
    \begin{subfigure}[b]{0.30\textwidth}
        \centering
        \includegraphics[width=\linewidth]{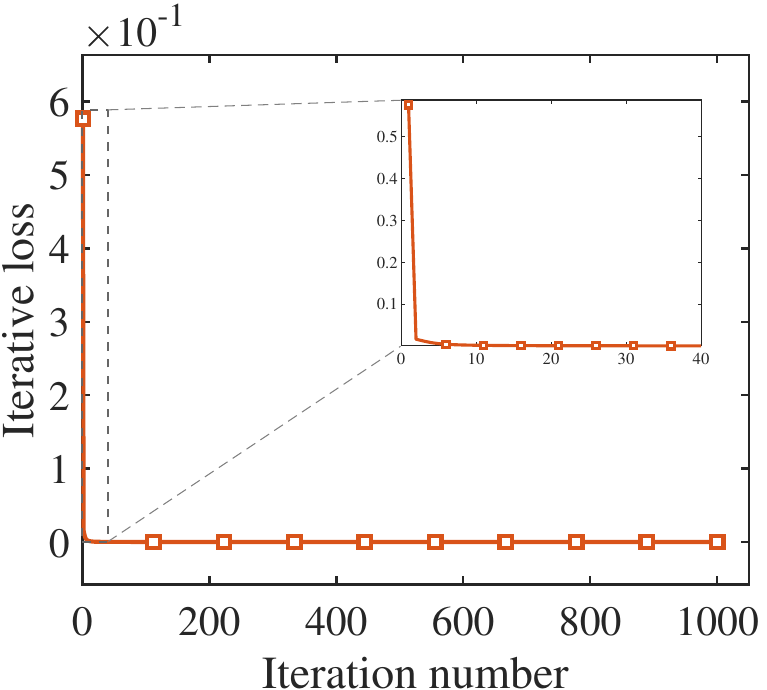}
        \caption{BCI} 
    \end{subfigure}
    \hfill    
    \begin{subfigure}[b]{0.30\textwidth}
        \centering
        \includegraphics[width=\linewidth]{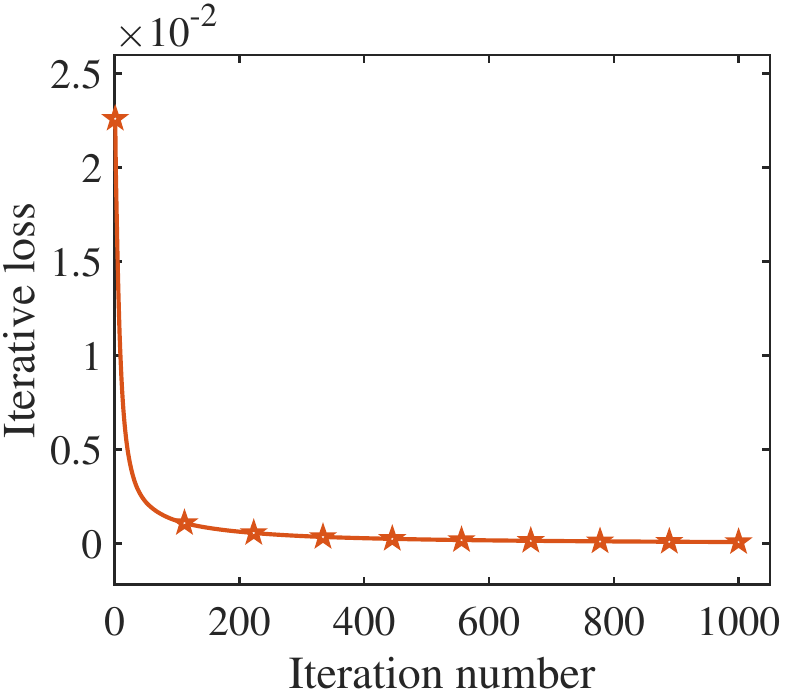}
        \caption{WDBC} 
    \end{subfigure}

    \caption{Convergence curves of the proposed algorithm.}
    \label{fig4}
\end{figure*}

\section{Conclusion}\label{sec5}

This paper proposes a novel low-rank classification model called HL-SMM, which innovatively integrates the Heaviside loss with an explicit low-rank constraint. Unlike the existing work, the Heaviside loss is highly robust to noise and outliers while preserving discriminative information from normal data. Furthermore, we establish the sufficient and necessary optimality conditions, thus designing an effective PAM algorithm where all subproblems have closed-form solutions. Numerical experimental results demonstrate the superior classification performance and noise robustness over state-of-the-art SMM and SVM variants.

In the future, we are interested in developing more efficient second-order optimization algorithms with rigorous sequence convergence. Besides, combining SMM with deep neural networks is a promising direction, which can not lnreduce dependence on hyperparameters and enhance the generalization ability \cite{liu2025star}.

\backmatter

\bibliography{my_references}

\end{document}